\PassOptionsToPackage{table}{xcolor}
\documentclass[arxiv]{meowreport}
\usepackage{xcolor}
\usepackage{soul}

\sethlcolor{cyan!30}
\usepackage{tcolorbox}
\tcbuselibrary{skins,breakable}

\usepackage{comment}    % \begin{comment} blocks inside the section files
\usepackage{wrapfig}    % wraptable floats in the experiment section
\usepackage{listings}   % VLM prompt listings in the appendix
\usepackage{amsthm}     % definition / proposition / proof in the appendix

\newtheorem{definition}{Definition}
\newtheorem{proposition}{Proposition}

\renewcommand{\ie}{i.e.\ }

\newcommand{\ML}[1]{\textcolor{purple!80!black}{\textbf{[ML:} #1\textbf{]}}}

\newcommand{\C}[1]{\textcolor{blue!80!black}{\textbf{[Chi:} #1\textbf{]}}}

\colorlet{punct}{red!60!black}
\definecolor{background}{HTML}{EEEEEE}
\definecolor{delim}{RGB}{20,105,176}
\colorlet{numb}{magenta!60!black}
\lstdefinelanguage{json}{
    basicstyle=\normalfont\ttfamily,
    numbers=left,
    numberstyle=\scriptsize,
    stepnumber=1,
    numbersep=8pt,
    showstringspaces=false,
    breaklines=true,
    frame=lines,
    backgroundcolor=\color{background},
    literate=
     *{0}{{{\color{numb}0}}}{1}
      {1}{{{\color{numb}1}}}{1}
      {2}{{{\color{numb}2}}}{1}
      {3}{{{\color{numb}3}}}{1}
      {4}{{{\color{numb}4}}}{1}
      {5}{{{\color{numb}5}}}{1}
      {6}{{{\color{numb}6}}}{1}
      {7}{{{\color{numb}7}}}{1}
      {8}{{{\color{numb}8}}}{1}
      {9}{{{\color{numb}9}}}{1}
      {:}{{{\color{punct}{:}}}}{1}
      {,}{{{\color{punct}{,}}}}{1}
      {\{}{{{\color{delim}{\{}}}}{1}
      {\}}{{{\color{delim}{\}}}}}{1}
      {[}{{{\color{delim}{[}}}}{1}
      {]}{{{\color{delim}{]}}}}{1},
}
\renewenvironment{abstract}{\begin{meowabstract}}{\end{meowabstract}}

\meowtitle{ElasticTTT: Prior-Preserving Test-Time Tuning for Video Editing}
\meowauthors{  Yueyi Liu$^{1*}$ \quad
    Chi Zhang$^{1*}$ \quad
    Sen Cui$^{2}$ \quad
    Miao Liu$^{1\dagger}$}
\meowaffiliation{College of AI, Tsinghua University$^{1}$,

Beijing Academy of Artificial Intelligence$^{2}$

}
\meowdate{\today}
\meowversion{Technical report}
 \meowlinks{
   \meowlink{Project Page}{https://liuyueyi-del.github.io/ElasticTTTprojectpage/}
   \quad
   \meowlink{Code}{https://github.com/liuyueyi-del/ElasticTTT-for-video-editing}
   \quad
   \meowlink{Dataset}{https://huggingface.co/datasets/liuyueyi-8/ElasticTTT-video-editing-dataset}
}

\begin{document}

\makemeowtitle
\renewcommand{\thefootnote}{}
\footnotetext{$^*$ Equal contribution}
\footnotetext{$^\dagger$ Corresponding author}
\renewcommand{\thefootnote}{\arabic{footnote}}
\begin{abstract}

Test-Time Tuning (TTT) on pretrained diffusion models has emerged as a powerful paradigm for video editing. However, there exists a foundational mismatch between the distribution-mapping nature of generative models and the single-point optimization of standard TTT. In this paper, we demonstrate that this mismatch triggers \textit{Prior Collapse}, a degenerate state where the model discards the text conditions and spatial latents, collapsing generations to the source video, or entangling the features of distinct regions. To resolve this, we propose \textbf{ElasticTTT}, a novel framework that preserves the prior generative distribution and rescues generative elasticity. Specifically, we propose \textit{Target Distribution Regularization} to prevent sharp memorization minima, \textit{Contrastive CFG} to guide inference away from source biases, and \textit{Asynchronous Noise Schedule} to preserve unedited regions. Extensive evaluations, supported by theoretical analysis, demonstrate that ElasticTTT successfully preserves the generative prior of the base model, achieving state-of-the-art performance on one-shot video editing.

\end{abstract}
\section{Introduction}

Advancements in denoising diffusion models~\cite{ho2020denoising, song2020score} have enabled significant progress in text-to-video generation~\cite{kong2024hunyuanvideo,wan2025wan,hong2022cogvideo,yang2025cogvideox}. These powerful video generation backbones also open new opportunities for controllable video editing by leveraging learned generative priors~\cite{DBLP:journals/corr/abs-2503-11412,mou2024revideo,huang2025dive}. However, user-provided source videos may fall outside the fixed training distribution of off-the-shelf models~\cite{yu2025veggie,bai2025ditto, zi2025se, wu2025insvie, he2025openve,li2025qffusion, bian2025videopainteranylengthvideoinpainting,zhang2025voicebridge}, particularly given the rapid emergence of new concepts across the internet, resulting in degraded editing performance. Recently, another line of research has explored Test-Time Tuning (TTT)~\cite{sun2020test, wang2020tent,wang2025low,wu2023tune, huang2025dive, Harsha_2024_CVPR, wang2025t3v2v}, which aims to adapt the weights of off-the-shelf models to the domain of user-provided source videos.

% , allowing more consistent and faithful video editing performance without any large scale training cost.

Compared with video editing approaches that perform inference directly with pre-trained models, TTT is particularly well suited to challenging scenarios where the source video exhibits a large domain gap from the pretraining distribution. This is because TTT explicitly adapts the model to the specific content and dynamics of each source video, thereby improving the preservation of its appearance, structure, and motion. Besides serving as a practical paradigm for video editing, TTT-based methods also function an effective data engine for synthesizing paired editing data to train large-scale feed-forward video editing models~\cite{wu2025insvie,zi2026senorita}.

\begin{figure}[t]
    \centering
    \includegraphics[width=1.0\linewidth]{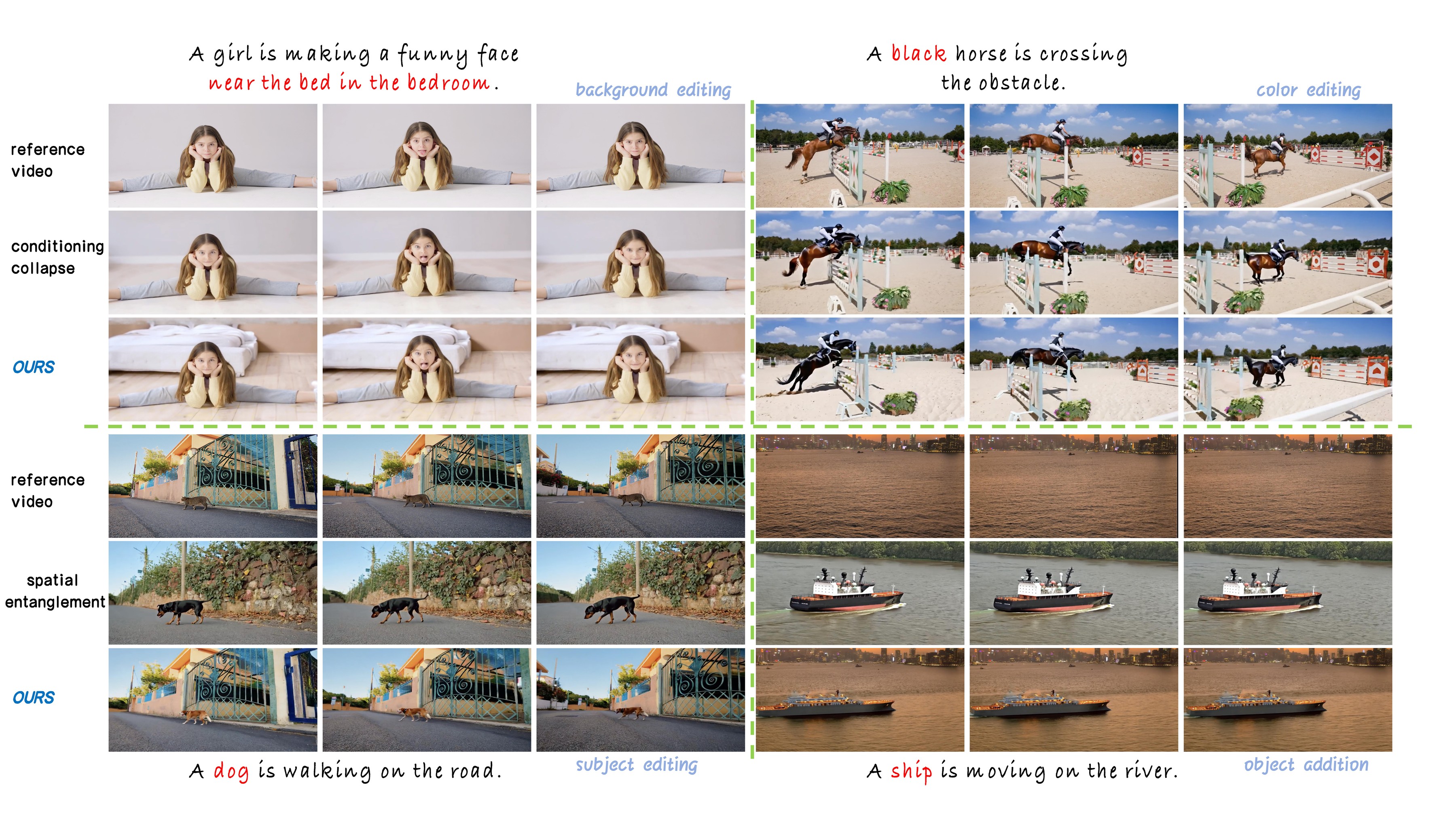}
    \vspace{-8mm}
    \caption{Demonstration of ElasticTTT and the baseline TTT method on four different editing tasks. ElasticTTT obtains high quality editing results among all tasks, whereas the baseline method exhibits clear Conditioning Collapse and Spatial Entanglement.}
    \label{fig:demo}
    \vspace{-3mm}
\end{figure}

Although TTT provides a natural mechanism for personalized user experiences~\cite{ruiz2023dreambooth, bi2025customttt, gal2022image, kumari2023multi,wu2026pipbench}, it introduces an inherent tension with the diffusion process: TTT intentionally encourages instance-specific overfitting, while diffusion processes are inherently stochastic and intended to preserve distributional diversity. Notably, optimizing a highly parameterized diffusion model exclusively on this zero-variance singularity can easily lead to a failure mode that we term \emph{prior collapse}. As illustrated in~\cref{fig:demo}, prior collapse manifests itself in two complementary failure phenomena. \emph{Conditioning Collapse} arises when the conditional guidance becomes anchored to the source-conditioned dynamics during tuning, preventing effective semantic control and causing the model to disregard the target editing instruction. \emph{Spatial Entanglement} arises when spatial representations become globally coupled during denoising, preventing localized control and causing unintended modifications outside the target region.

To circumvent these challenges, previous methods have proposed a spectrum of interventions. To mitigate Conditioning Collapse, prior works often rely on low-rank adaptation (LoRA)~\cite{gao2025lora}, prior-preservation losses~\cite{ruiz2023dreambooth}, self-supervised objectives~\cite{wang2025low}, or the isolated optimization of text and null-text embeddings~\cite{gal2022image, mokady2023null}. To resolve Spatial Entanglement, methods typically resort to invasive attention map modifications~\cite{kwon2024unified,zuo2025cut}, masked optimization phases~\cite{gao2025lora}, or the injection of external control signals like depth and optical flow~\cite{wang2025t3v2v}.
However, these approaches are fundamentally sub-optimal, their complexity determines that they heavily rely on delicate architectural tuning, require auxiliary datasets, or remain dangerously sensitive to hyperparameter choices.

More importantly, while Test-Time Tuning (TTT) has shown immense promise in image editing by fine-tuning models directly on the source image, adapting TTT to video is remarkably scarce and fraught with difficulty. Video Diffusion Transformers~\cite{wan2025wan,peebles2023scalable} possess highly complicated spatial-temporal priors, and naively transferring image-based TTT optimization techniques to video models severely accelerates Prior Collapse. The network quickly overfits to the deterministic motion and appearance of the source video, losing its generative elasticity required for precise, accurate editing.

To resolve this fundamental optimization–sampling tension, we propose \textbf{ElasticTTT}, a principled framework designed to rescue generative elasticity and explicitly cure the symptoms of Prior Collapse. Rather than treating these failures as isolated artifacts, ElasticTTT intervenes across the entire generative pipeline: First, to prevent Conditioning Collapse during optimization, we introduce \emph{Target Distribution Regularization}. By softly relaxing the rigid, single-point target with controlled stochasticity, we physically prevent the model from memorizing the source video, preserving its elastic denoising priors. However, as the tuned weights still exhibit a residual shift toward the source concept, we subsequently apply \emph{Contrastive Classifier-Free Guidance} during inference to forcefully steer the trajectory toward the target condition. Finally, to resolve Spatial Entanglement, we introduce an Asynchronous Noise Schedule in the sampling process, which dynamically assign independent diffusion trajectories to the edited and anchored regions, clarifying latent representations and granting the model the localized freedom to render new concepts while safely preserving the unedited surroundings.
% \ML{please check, still sound  abit piling things up}

To validate the effectiveness of our framework, we conduct extensive experiments across a diverse range of editing tasks. Our results demonstrate that ElasticTTT successfully resolves the prior collapse issue, exhibiting robust generative elasticity, precise instruction adherence, and high-fidelity preservation of the source video. In both quantitative and qualitative evaluations, ElasticTTT consistently outperforms competitive baselines—spanning inference-only, test-time-tuning, and training-based approaches—ultimately achieving state-of-the-art performance in video editing.

\section{Related work}
\subsection{Test-Time Tuning of Diffusion Models}

Test-Time Tuning (TTT)~\cite{sun2020test, wang2020tent} updates a pre-trained diffusion model at inference time for tasks like customized and personalized generation~\cite{ruiz2023dreambooth, kumari2023multi, gal2022image, bi2025customttt}, stylized generation~\cite{zhang2023inversion, chung2024style, sohn2023styledrop, wang2023stylediffusion}, and test-time scaling~\cite{ma2025inference}.  For generative editing, TTT establishes a high-fidelity anchor by reconstructing a specific reference video~\cite{wu2023tune} for subsequent modifications, but optimizing the model on a single reference naively often degenerates performance~\cite{ruiz2023dreambooth}. Previous works attempt to mitigate this issue via low-rank adaptation~\cite{gao2025lora, gal2022image}, additional prior-preserving losses~\cite{ruiz2023dreambooth}, only optimizing text embedding~\cite{gal2022image} or null-text embeddings~\cite{mokady2023null}, or introducing self-supervised signals~\cite{wang2025low}. However, these interventions do not fully resolve the issue, especially on highly parameterized text-to-video models~\cite{wan2025wan}, often resulting in a zero-sum compromise between following novel editing instructions and preserving details of the source video.

\subsection{Video Editing}
Video editing requires seamlessly integrating new semantic concepts or modifying existing features while preserving the source video's unedited structures. Existing methods generally fall into three paradigms. \textit{Training-based methods}~\cite{yu2025veggie,bian2025videopainteranylengthvideoinpainting,DBLP:journals/corr/abs-2503-11412,li2025qffusion, wu2025insvie, he2025openve, bai2025ditto, zi2025se} fine-tune video generation models on paired datasets, but suffer from severe data scarcity and limited generalization on new concepts, and heavy computational resources needed for training. \textit{Training-free methods}~\cite{kulikov2025flowedit,cohen2024slicedit,DBLP:journals/corr/abs-2403-16111, ku2024anyvv, geyer2024tokenflow, jeong2024groundavideo, bai2025uniedit,liu2025stablev2v} bypass data requirements by leveraging pre-trained priors via inverse diffusion methods~\cite{song2020denoising, meng2021sdedit}, often including manipulating attention maps~\cite{meral2024motionflow,liu2025stablev2v,kwon2024unified} or utilizing pretrained image editing models~\cite{DBLP:journals/corr/abs-2405-16537,gao2025lora,ku2024anyvv,liu2025stablev2v}, but inherently struggle to balance semantic alignment and reference preservation due to fixed model parameters and sampling trajectory length. \textit{Test-Time Tuning (TTT) methods}~\cite{wu2023tune,huang2025dive,Harsha_2024_CVPR,liu2024video} bridge this gap by tuning the base model directly on the source video, improving source reconstruction without involving external datasets. However, adopting TTT for video editing directly encounters the prior collapse issue, to which we offer a systematic solution in our work. More details of the related works are shown in Supplementary.

\begin{figure}[t]
    \centering
    \includegraphics[width=1\textwidth]{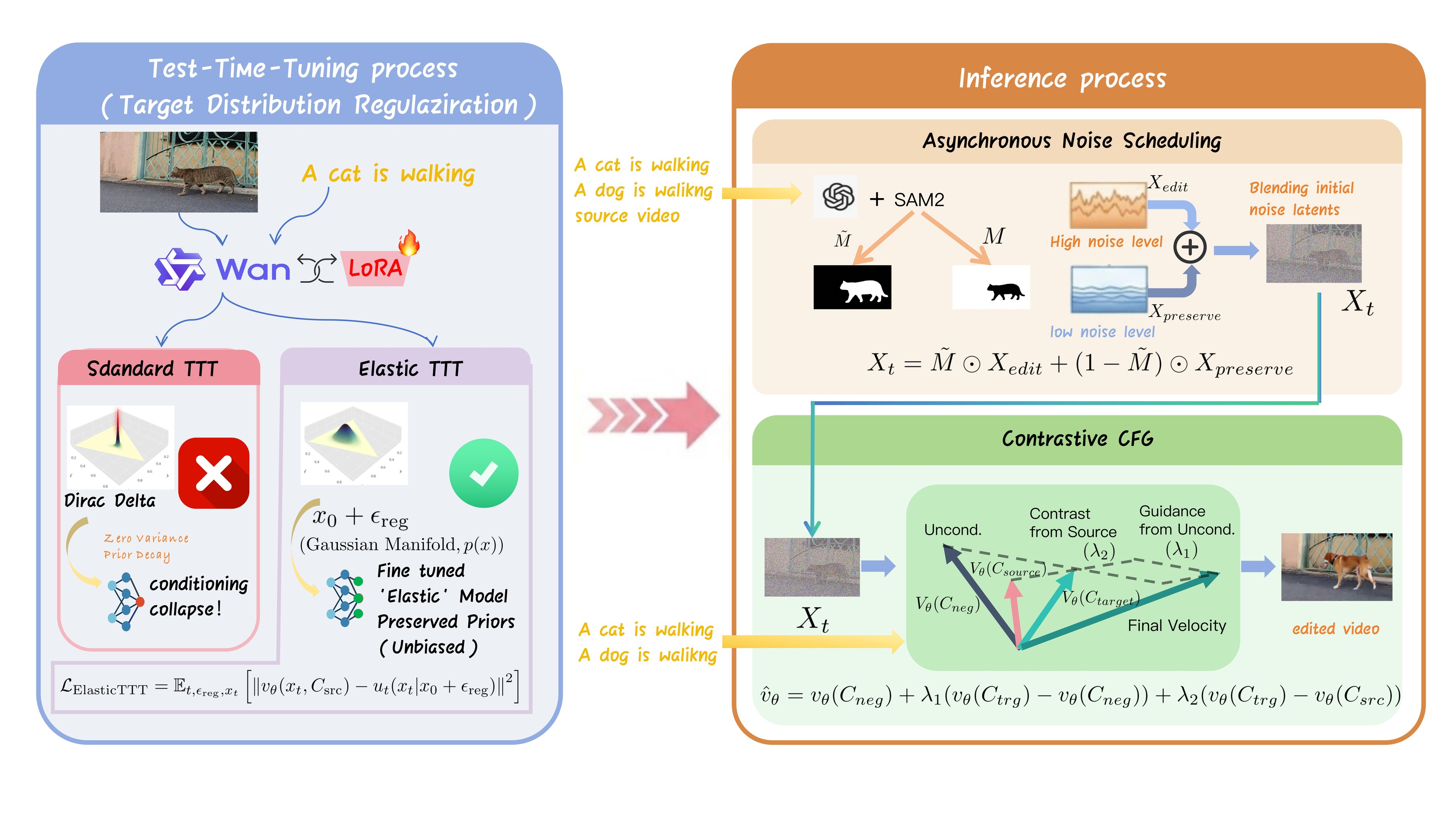}
    \vspace{-9mm}
    \caption{\textbf{The overall pipeline of ElasticTTT}. To resolve Conditioning Collapse, we soften the optimization target via TDR during training, and repel the sampling process from the source distribution via Contrastive CFG during sampling. To mitigate Spatial Entanglement, we perform Asynchronous Noise Scheduling to establish independent integration trajectories for distinct regions.}
    \label{fig:pipeline}
    \vspace{-3mm}
\end{figure}

\section{Method}

\label{sec:overview}

Given a source video $V_{\text{src}}$ and its corresponding descriptive text prompt $C_{\text{src}}$, the goal of video editing is to generate a target video $V_{\text{trg}}$ that aligns with a user-provided target prompt $C_{\text{trg}}$ while preserving the fundamental structural and temporal dynamics of $V_{\text{src}}$. Standard video editing models sample from a fixed conditional distribution 
$p_{\theta}(V_{\text{trg}}\mid V_{\text{src}}, C_{\text{trg}})$ with frozen parameters $\theta$, 
whereas test-time tuning approaches first adapt the model parameters on the source input via 
$\theta' = \mathcal{A}(\theta; V_{\text{src}}, C_{\text{src}})$, 
and then sample from the instance-adapted distribution 
$p_{\theta'}(V_{\text{trg}} \mid V_{\text{src}}, C_{\text{trg}})$, 
allowing stronger identity preservation and scene-specific consistency under the desired editing prompt.

Concretely, in the first stage, the model operates in a latent space and is optimized to reconstruct the source video latent $x_0$ by minimizing the MSE loss between the predicted velocity $v_\theta$ and the target conditional vector field $u_t$ given the noised latent $x_t$ at timestep $t$:
\begin{align}
    \mathcal{L}_{\text{TTT}} = \mathbb{E}_{t, x_t} \left[ \lVert v_\theta(x_t, t, C_{\text{src}}) - u_t(x_t \mid x_0) \rVert^2 \right]
\end{align}
Once the model is fine-tuned, inference is performed starting from a noisy latent, acquired by DDIM inversion~\cite{song2020denoising}, utilizing the target prompt $C_{\text{trg}}$ to guide the generation of $V_{\text{trg}}$.

\subsection{The Optimization–Sampling Tension}
While standard TTT captures the spatial-temporal features of the reference video, it introduces a fundamental mismatch with the generative nature of diffusion models.
Diffusion models learn mappings between data distributions and rely on diverse samples to preserve generative flexibility. In TTT, however, the distribution of training data collapses to a single instance, effectively degenerating into a Dirac delta distribution, i.e., $p(x)=\delta(x-x_0)$. 

Optimizing a highly parameterized diffusion model on such a single-point distribution can easily lead to a failure mode that we term \textbf{\textit{prior collapse}}: The model overfits to high-frequency pixel-level details of the source video, prioritizing exact reconstruction over generalizable condition-aware denoising. We conducted an illustrative experiment to examine this phenomenon, as shown in~\cref{fig:toy}: The base diffusion model is trained on a distribution with 25 classes (A), and TTT is performed on one sample of the center class. When we conduct inference with other class conditions, the model simply bypasses the conditions and generates a chaotic cluster around the memorized target distribution (B).

Formally, in video diffusion models, the network predicts the velocity field $v_\theta(x_t, t, C_\text{src})$ based on the noisy latent $x_t$ and the text condition $C_\text{src}$. Although the large pretrained model does not fully collapse the prior distribution, it still progressively learns to bypass meaningful variations in $x_t$ and $C_\text{src}$, instead hard-coding the high-frequency details of the source video into its parameters, degrading editing performance across two modalities:

\begin{itemize}
    \item \textbf{Conditioning Collapse} By repeatedly updating on a static text prompt $C_\text{src}$, the model gradually weakens the text condition pathways. This causes the model to ignore the target prompt $C_\text{trg}$ during inference, losing its semantic elasticity, anchoring to the source concept rather than following new textual guidance (\cref{fig:demo} first row).
    \item \textbf{Spatial Entanglement} By increasingly bypassing the noisy input $x_t$ to memorize the global target, the model degrades its ability to extract semantically meaningful spatial representations from $x_t$. Consequently, distinct semantic objects lose their boundaries, causing severe visual corruption and entanglement when attempting to edit specific regions (\cref{fig:demo} second row).
\end{itemize}
\begin{figure}[tbp]
    \centering
    \vspace{-3mm}
    \includegraphics[width=1\linewidth]{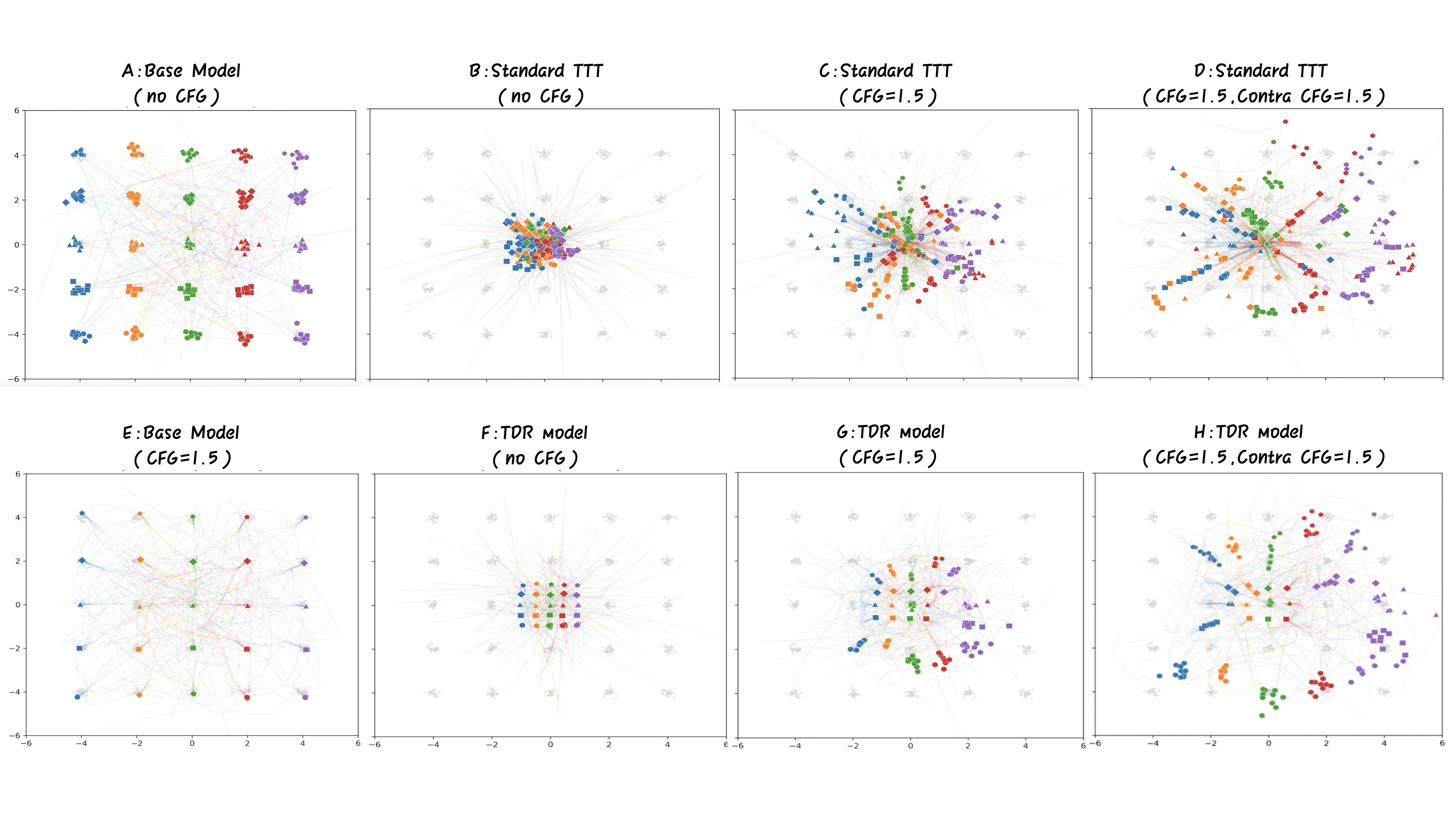}
    \vspace{-8mm}
    \caption{A 2D toy experiment, illustrating the effect of TDR and Contrastive CFG. In standard TTT, the generated distribution of the base model (A) collapses around the optimization target (B), and dispersing to chaos when CFG is applied (C, D). With TDR, the generated distribution remains its original structure (F), and Contrastive CFG further cancels out the introduced distribution shift (G, H).}
    \label{fig:toy}
    \vspace{-5mm}
\end{figure}
Importantly, prior collapse is not equivalent to overfitting: TTT intentionally encourages controlled overfitting to capture source characteristics, while prior collapse represents a degenerate regime in which generative behavior collapses and editing controllability is lost. 

To resolve this foundational mismatch, we propose \textbf{ElasticTTT}, a novel pipeline that preserves the generative prior. We mitigate Conditioning Collapse during training by softening the optimization target into a continuous local distribution~\cref{subsec:TDR}, and resolve the remaining conditioning bias induced through repelling the generation from the source condition through a contrastive classifier-free-guidance~\cref{subsec:contrastivecfg}. We resolve Spatial Entanglement during inference by dynamically processing different regions with asynchronous noise levels~\cref{subsec:asyncns}. The following sections detail the formulations of these three methods.

\subsection{Target Distribution Regularization: Escaping the Dirac Singularity}
\label{subsec:TDR}
The root cause of Prior Collapse of the standard TTT lies in the rigid and deterministic nature of the Dirac delta distribution $p_{data}(x) = \delta(x - x_0)$ of the target data. To prevent the generative prior from collapsing into this sharp singularity, we introduce \textit{Target Distribution Regularization} (TDR).

The core insight of TDR is to inject controlled stochasticity \textit{strictly} into the optimization target, while keeping the input trajectory completely clean. Specifically, given a clean source video $x_0$ and a standard Gaussian noise sample $x_1 \sim \mathcal{N}(0, I)$, we construct the intermediate noisy latent $x_t$ exactly as in standard flow matching $x_t = tx_1 + (1-t)x_0$. 
% However, when computing the target velocity that the network must predict, we construct a perturbed target $\tilde{x}_0$ instead:
However, when computing the target velocity for supervision, we replace the deterministic target $x_0$ with a perturbed version $\tilde{x}_0$: 
\begin{align} \tilde{x}_0 = x_0 + \epsilon_\text{reg}, \quad \epsilon_\text{reg} \sim \mathcal{N}(0, \sigma_\text{reg}^2 I) \end{align}
We then have the following updated optimization target $u_t(x_t|\tilde{x}_0) =  x_1 - \tilde{x}_0$. Note that the model input $x_t$ is interpolated using the original, unperturbed $x_0$ and is completely independent of $\epsilon_\text{reg}$.

Crucially, this stochastic injection does not alter the fundamental optimization objective of the model. In the original flow-matching framework~\cite{lipman2023flow}, the network is trained to predict the conditional expectation of the velocity field:
\begin{align} v^*(x_t, t) = \underset{v}{\arg\min} \, \mathbb{E}_{t, x_0, x_1} \left[ \| v(x_t, t) - u_t(x_t \mid x_0) \|^2 \right] = \mathbb{E}_{x_0 \mid x_t} \left[ u_t(x_t \mid x_0) \right] \end{align}
Because standard flow matching paths construct the velocity field to be linear with respect to the target data (\ie, $u_t(x_t|\tilde{x}_0) = x_1 - \tilde{x}_0 = x_1 - x_0 -\epsilon_\text{reg} - x_t = u_t(x_t|x_0)-\epsilon_\text{reg}$), training the network $v_\theta$ to predict this perturbed velocity field causes the new optimal vector field $\tilde{v}^*(x_t, t)$ to evaluate to:
\begin{align}
\tilde{v}^*(x_t, t) &= \mathbb{E}_{x_0, \epsilon_\text{reg} \mid x_t} \left[ u_t(x_t \mid \tilde{x}_0) \right] = \mathbb{E}_{x_0, \epsilon_\text{reg} \mid x_t} \left[ u_t(x_t \mid x_0) - \epsilon_\text{reg} \right]\\
&= \mathbb{E}_{x_0 \mid x_t} \left[ u_t(x_t \mid x_0) \right] - \mathbb{E}_{\epsilon_\text{reg}} \left[ \epsilon_\text{reg} \right] = v^*(x_t, t)
\end{align}
The injected noise $\epsilon_\text{reg}$ is sampled independently with zero mean, ensuring that the perturbation remains unbiased in expectation. Therefore, the global minimum of our regularized objective remains perfectly aligned with the original video's true vector field.

As shown in the toy experiment in~\cref{fig:toy}, instead of generating a collapsed distribution (B), the generation of the TDR model preserves the structures of the prior distributions (F), highlighting the significance of the introduced stochasticity.
Although the expected target remains identical, the variance $\sigma_\text{reg}^2$ fundamentally alters the optimization dynamics. By replacing the strict Dirac delta point with a continuous local Gaussian manifold, we introduce a geometry-aware variance penalty into the loss landscape. The model thus no longer minimizes the loss by lazily memorizing an infinitely sharp, rigid spatial mapping to $x_0$. Instead, it learns to retain its elastic properties as a denoising model and seeks a broader, generalized minimum. Consequently, the model robustly follows target text prompts even after extensive fine-tuning steps, dramatically increasing editing flexibility and stabilizing the TTT process. A more mathematical explanation is shown in Supplementary~\cref{app:math}.

\subsection{Contrastive Classifier-Free-Guidance: }
\label{subsec:contrastivecfg}
In~\cref{fig:toy} (F), after TDR, although the structure of the prior distribution is preserved, the generated points are still drawn closer to the source distribution trained in TTT. The same phenomenon occurs in video diffusion models: after TDR, the model no longer bypass $C_\text{src}$ and memorize the reference video entirely, yet the generated distribution still shifts toward the source distribution, showing a tendency of rendering the concepts introduced in the reference video.

To further cancel out the distribution shift and decouple the concepts of the source video which are planted into the model during training, we propose \textit{Contrastive Classifier-Free-Guidance}, an enhanced Classifier-Free-Guidance (CFG) which actively utilizes the source condition as a negative constraint.
The standard CFG~\cite{ho2022classifier} attempts to steer generation by creating a distribution shift between conditioned and unconditional predictions:
\begin{align}
    \hat{v}_\theta(x_t,t, C_{\text{trg}}) = v_\theta(x_t,t, C_{\text{neg}}) + \lambda\left( v_\theta(x_t,t, C_{\text{trg}}) - v_\theta(x_t,t, C_{\text{neg}}) \right)
\end{align}
where $C_{\text{neg}}$ is typically an empty or generic negative prompt, and $\lambda$ is the guidance scale. 
To further suppress the embedded optimization bias, we expand this guidance into a tri-directional formulation. We explicitly contrast the target against the source, establishing a semantic opposition between them, by introducing a secondary negative constraint leveraging the original source prompt $C_{\text{src}}$:
\begin{align}
    \hat{v}_\theta(x_t,t, C_{\text{trg}}) &= v_\theta(x_t,t, C_{\text{neg}}) + \lambda_1 \left( v_\theta(x_t,t, C_{\text{trg}}) - v_\theta(x_t,t, C_{\text{neg}}) \right) \nonumber \\&+ \lambda_2 \left( v_\theta(x_t,t, C_{\text{trg}}) - v_\theta(x_t,t, C_{\text{src}}) \right)
\end{align}
where $\lambda_1$ controls the standard unconditional guidance and $\lambda_2$ governs the strength of the semantic contrast between the source and target conditions.

By contrasting the two conditions, Contrastive CFG isolates and amplifies the semantic delta between them, repelling the sampling destination away from the source distribution. As demonstrated in~\cref{fig:toy}, applying standard CFG (G)—which merely pulls the velocity toward the target condition—yields only a marginal shift in the generated distribution due to the residual gravitational pull of the memorized source. In contrast, our Contrastive CFG (H) actively repels the generation process from the source concept. This active contrast guarantees crisp semantic transitions and fully restores the model's generative elasticity.

\begin{figure}[t]
    \centering
    \includegraphics[width=1\linewidth]{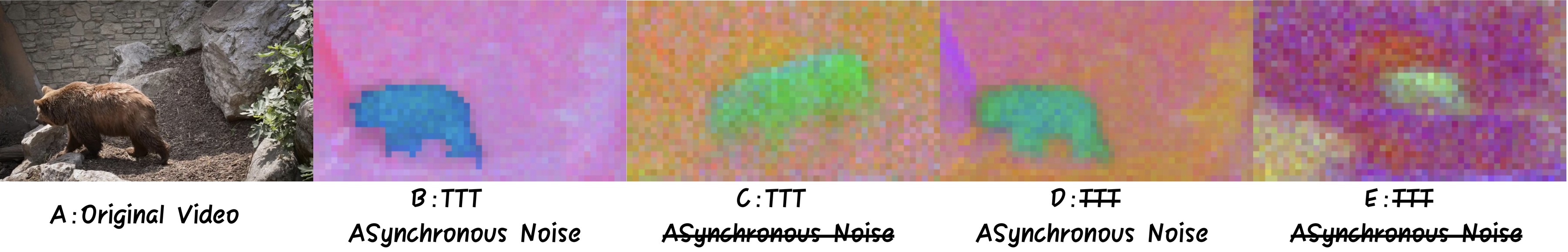}
    \caption{PCA~\cite{hotelling1933analysis} visualization of the DiT latent representations. When TTT is applied without Async-NS, the representation blurs and the model fails to capture the reference video's structure. Async-NS enables the model to extract clear, aligned representation}
    \label{fig:spatial-entangle-visualize}
    \vspace{-4mm}
\end{figure}

\subsection{Asynchronous Noise Scheduling}
\label{subsec:asyncns}
While TDR and Contrastive CFG cure Conditioning Collapse, the fine-tuned network remains vulnerable to Spatial Entanglement. Because the tuned network learns to bypass the noisy input $x_t$ and directly memorize $x_0$, it fails to extract localized, semantically rich spatial representations. As visualized in~\cref{fig:spatial-entangle-visualize} (C), during inference, the representation becomes blurred and misaligned with the source video, with the foreground and background heavily entangled.

Many existing approaches attempt to alleviate Spatial Entanglement by modifying attention maps~\cite{kwon2024unified,zuo2025cut}, using pre-edited images as guide~\cite{ku2024anyvv,gao2025lora,liu2025stablev2v} or intervening the optimization phase with masks~\cite{gao2025lora,wang2025low}. However, these complicated methods often compromise the global structural coherence. In contrast, we propose \textit{Asynchronous Noise Scheduling} (Async-NS), an inference-time strategy that cleanly circumvents Spatial Entanglement by decoupling the integration trajectories of edited and anchored regions.

Our core insight is that temporal uniformity exacerbates Spatial Entanglement. Evaluating the entire latent at a single timestep $t$ allows the network to fall back on its memorized output. By explicitly assigning distinct noise levels and time embeddings to the edited versus anchored regions, we physically break the entanglement of representations. As evidenced in~\cref{fig:spatial-entangle-visualize} (B), this temporal asymmetry forces the network to process the regions distinctly, instantly restoring clear representational boundaries. Furthermore, this decoupling inherently aligns with the distinct generative needs of the regions: the edited area receives a high-noise regime ($T_\text{e}$) for the flexibility to render new concepts, while the preserved region is restricted to a low-noise regime ($T_\text{p}$) that carries the maximum information from the source video.

We define a frame-by-frame mask $M$, either from manual input for accuracy or from an automatic segmentation model for convenience, to mark the edited areas. $M$ is then downsampled to match the latent resolution and spatially smoothed into a continuous transition map $\tilde{M} \in [0, 1]$ to naturalize the concept boundaries in latent space. We define two asynchronous noise schedules: $\{t^i_\text{e}\}_{i=1}^N$ and $\{t^i_\text{p}\}_{i=1}^N$, with $t^N_{(\cdot)} = T_{(\cdot)}$ and $t^1_{(\cdot)} = 0$. The sampling starts from an initial latent $\hat{x}^N$, a noisy version of the source video latent with a mixed noise level, and at each step $i$, the asynchronous updates are spatially fused:
\begin{align}
\hat{x}^N &= \tilde{M}\odot x_{T_\text{e}} + (1-\tilde{M})\odot x_{T_\text{p}}\nonumber\\
\hat{x}^{i-1} &= \tilde{M} \odot \left( \hat{x}^i - \hat{v} \Delta t^{i}_\text{e} \right)+ (1 - \tilde{M}) \odot \left( \hat{x}^i - \hat{v} \Delta t^{i}_\text{p} \right)
\end{align}
where $\hat{v}$ denotes the final velocity by Contrastive CFG predicted with a spatially mixed timestep embedding of $t^i_\text{e}$ and $t^i_\text{p}$ as input: 
\begin{align}
    \hat{v}=\hat{v}_\theta(\hat{x}^i, t^i,C_\text{trg}),\quad t^i =\tilde{M}\odot t^i_\text{e}+(1-\tilde{M})\odot t^i_\text{p}
\end{align} 

By breaking the global temporal synchronization, Async-NS provides the model with a spatially clear, disentangled latent representation with distinct noise levels and timestep embeddings. Meanwhile, this formulation also grants the model localized freedom to aggressively overwrite target concepts along a longer integration path, while safely guiding unedited surroundings back to their exact physical state along a restricted, deterministic path.

\section{Experiment}

\begin{table}[t]
  \caption{\textbf{Quantitative comparison with prior video editing approaches.} Best results are highlighted in \textcolor{red}{red}, and second-best in \textcolor{blue}{blue}. * refers that the method is re-implemented with our configurations.
  }
  
  \centering
  \label{tab:video_editing_comparison}

  \renewcommand{\arraystretch}{1.2}
  \setlength{\tabcolsep}{5pt}
  \begin{tabular}{lccccccc}
    \toprule
    Methods &VBench$\uparrow$& CLIP-T$\uparrow$ & VEBench$\uparrow$&VQ$\uparrow$&IA$\uparrow$&SP$\uparrow$&OVL$\uparrow$ \\
    \midrule
    \multicolumn{8}{l}{\textit{Other Methods}} \\ 
     AnyV2V\cite{ku2024anyvv} & 4.80 &\textcolor{blue}{30.21} &0.53 &3.51&5.89&2.30&4.31\\
     Ground-A-Video\cite{jeong2024groundavideo} & 4.89 &28.51 &0.13&3.57&4.49&2.42& 3.90\\
     Token-flow\cite{geyer2024tokenflow} &4.82 &29.69 &\textcolor{blue}{0.67} &4.47&5.57&4.30&4.83\\
    
     Ditto \cite{bai2025ditto}& 4.88& 28.93&0.60&\textcolor{blue}{5.48}&5.95&3.32&\textcolor{blue}{5.62}\\
     Flow-align \cite{kim2026flowalign}& 4.93&27.72& 0.61&5.28&5.44&\textcolor{red}{7.23}&5.44\\
     Uniedit \cite{bai2025uniedit}&  4.31&\textcolor{red}{31.76} &0.64  &3.87&\textcolor{blue}{6.51}&0.54&5.04\\
      \midrule
    \multicolumn{8}{l}{\textit{Test-Time-Tuning Methods}} \\ 
    Tune-A-Video* \cite{wu2023tune}&\textcolor{red}{5.00}&28.40&0.72&5.34&6.15&3.63&5.39\\
    VidTTA*\cite{wang2025low}&\textcolor{blue}{4.98}&28.47&0.72&5.00&5.76&4.42&5.08\\
    \rowcolor{yellow!30}ElasticTTT&\textcolor{blue}{4.98}& 29.58&\textcolor{red}{0.72} &\textcolor{red}{6.19}&\textcolor{red}{7.50}&\textcolor{blue}{5.23}&\textcolor{red}{6.68} \\
  \bottomrule
  \end{tabular}
    \vspace{-3mm}
\end{table}

\subsection{Experiment setup}
To validate the effectiveness and robustness of ElasticTTT, we conducted extensive experiments and compared it with a wide range of video editing models. We utilize Wan2.1 1.3B~\cite{wan2025wan} as the base model and conduct TTT with Low Rank Adaptation (LoRA)~\cite{hu2022lora} for 100 steps. We set the TDR variance to $\sigma_\text{reg} = 0.2$, the CFG guidance scales to $\lambda_1=6, \lambda_2 = 2$, and distinct noise regimes to $T_\text{p}=0.55, T_\text{e}=0.97$. The segmentation mask $M$ is automatically generated by a segmentation model Grounded-SAM2~\cite{ravi2024sam2} with a VLM Qwen3-VL-32B~\cite{Qwen3-VL} to specify the edited concept.
All experiments are done on a single Nvidia pro6000 GPU, where our full method introduces only a $\sim$7\% inference-time overhead over vanilla TTT (approximately 5\% from Contrastive CFG and 2\% from Async-NS).
More implementation details, along with a sensitivity analysis on these hyperparameters that verifies the robustness of our method, are provided in Supplementary~\cref{exp}.
% \begin{comment}
% \subsection{Experiment setup}
% To validate the effectiveness and robustness of ElasticTTT, we conducted extensive experiments and compared them with a wide range of video editing models. We utilize Wan2.1 1.3B~\cite{wan2025wan} as the base model and conduct TTT with Low Rank Adaptation (LoRA)~\cite{hu2022lora} for 100 steps. We set the TDR variance to $\sigma_\text{reg} = 0.2$, the CFG guidance scales to $\lambda_1=6, \lambda_2 = 2$, and distinct noise regimes to $T_\text{p}=0.55, T_\text{e}=0.97$. The segmentation mask $M$ is automatically generated by a segmentation model Grounded-SAM2~\cite{ravi2024sam2} with a VLM Qwen3-VL-32B~\cite{Qwen3-VL} to specify the edited concept.
% All experiments are done on a single Nvidia pro6000 GPU. More details on the experiment setup are listed in Supplementary.
% We also conduct an ablation study on these hyperparameters to prove the stability.
% ElasticTTT remains highly efficient: it introduces only a $\sim$7\% inference-time overhead relative to the vanilla TTT baseline, of which approximately 5\% is attributed to Contrastive CFG and 2\% to the asynchronous noise scheduling strategy.
% \end{comment}
\begin{figure}[!t]
    \centering
    \includegraphics[width=1\linewidth]{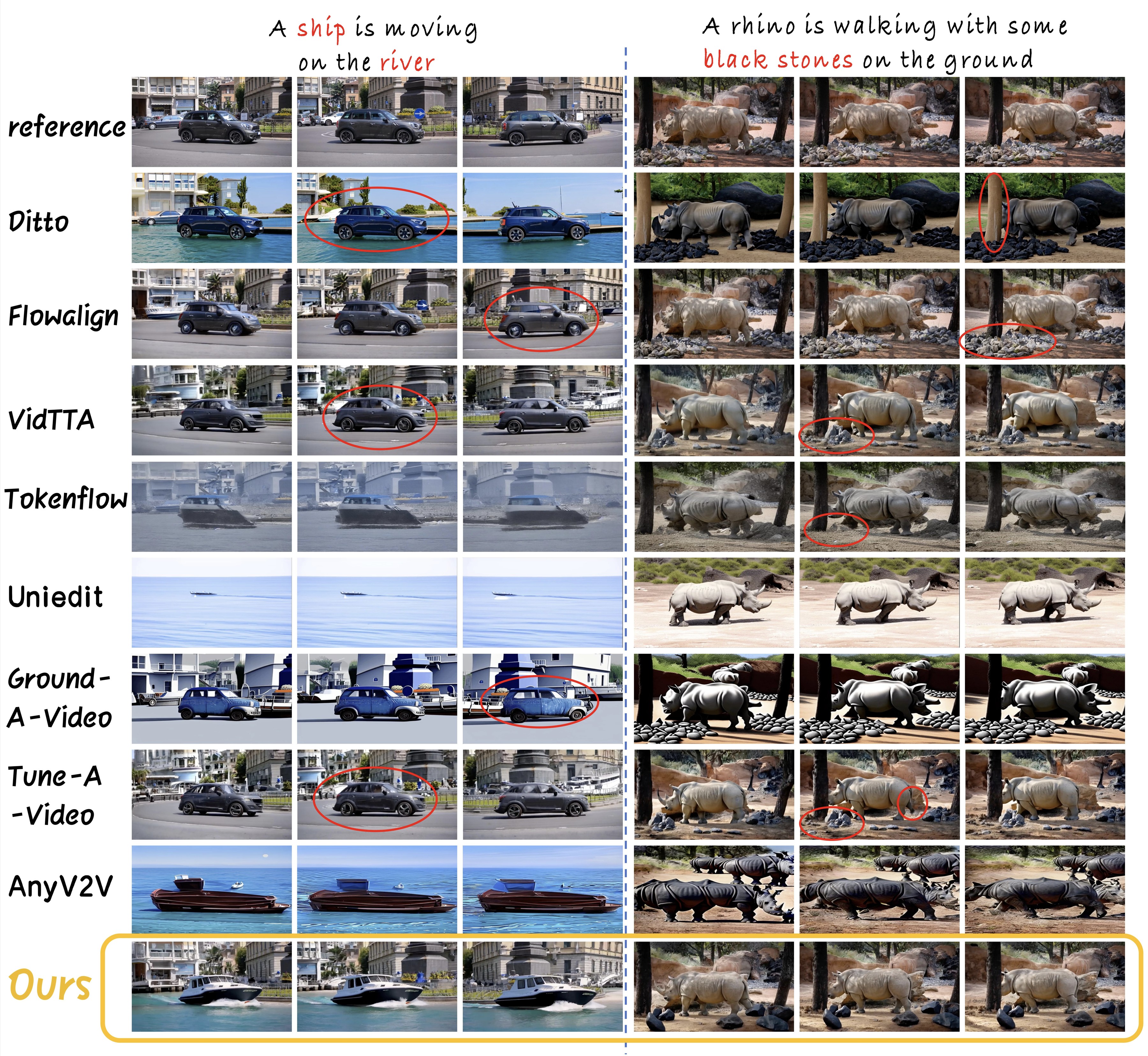}
    \caption{Visual comparisons with other methods. Our method strictly follows the editing instructions while successfully preserving the details in unedited regions.}
    \label{fig:qualitative_comparison}
    \vspace{-5mm}
\end{figure}

\subsubsection{Evaluation Datasets}
Prior works~\cite{ yang2025videograin,meral2024motionflow,kwon2024unified,liang2025looking,geyer2024tokenflow,liang2024flowvid} construct their testset by selecting text-video pairs from the Davis dataset~\cite{pont20172017,perazzi2016benchmark,caelles20182018}. However, either their data is not released~\cite{geyer2024tokenflow,kahatapitiya2024object} or the selection contains only limited editing tasktypes~\cite{liu2025stablev2v}. Thus, we choose to select a testset following the common protocol on our own, including 125 text-video pairs.
For comprehensive evaluation, our testset covers five different types of tasks, specifically: subject editing, background editing, subject addition, color adjustment, and general editing. We define general editing as reconstructing all visual elements in the video while preserving only the general motion dynamics of the reference subject. This editing mode is particularly valuable when the goal is to retain the cinematographic composition (or camera movement) rather than the visual content. We will release our selection to promote fair and comprehensive benchmarking for future works.

\subsubsection{Prior Methods}

We compare ElasticTTT against several strong video editing baselines: FlowAlign~\cite{kim2026flowalign}, UniEdit~\cite{bai2025uniedit}, Ditto~\cite{bai2025ditto}, VidTTA~\cite{wang2025low}, TokenFlow~\cite{geyer2024tokenflow}, AnyV2V~\cite{ku2024anyvv}, Tune-A-Video~\cite{wu2023tune} and Ground-A-Video~\cite{jeong2024groundavideo}. This selection represents a wide-range of conventional video editing paradigms, encompassing training-free methods~\cite{bai2025uniedit, geyer2024tokenflow, jeong2024groundavideo, kim2026flowalign}, first-frame-guided approaches~\cite{ku2024anyvv}, large-scale training-based models~\cite{bai2025ditto}, and other Test-Time Tuning (TTT) methods~\cite{wu2023tune, wang2025low}.
Notably, to ensure fair comparison with TTT-based baselines, we re-implement Tune-A-Video~\cite{wu2023tune} and VidTTA~\cite{wang2025low} using the \textit{exact same configurations} with ElasticTTT, including all settings of base model, training and inference.

\subsubsection{Evaluation Metrics}
We evaluate ElasticTTT and the baseline models using a comprehensive suite of objective and subjective metrics.

\emph{Automatic Metrics}.\ Following standard evaluation protocols~\cite{wang2025t3v2v,bai2025ditto, bai2025uniedit, wu2023tune}, we compute the CLIP similarity to measure semantic alignment between the edited video and the target prompt. In addition, we incorporate established metrics from video generation and editing benchmarks. We include metrics from VBench~\cite{huang2024vbench} to evaluate the overall visual quality and temporal consistency of the generated video. Specifically, we averaged 6 VBench metrics after normalizing them separately, to make a general assessment of the generated video. Also, we include the metric from VEBench~\cite{liu2025vebench} as a general indicator of video editing quality. These automatic metrics primarily serve as efficient and reproducible signals for model development, ablation analysis, and large-scale benchmarking.

\emph{Judge-Based Evaluation}.\ Despite their efficiency, traditional quantitative metrics often fail to capture nuanced editing requirements, such as fine-grained preservation in unedited regions or complex instruction adherence. To address this limitation, we leverage the advanced visual reasoning capabilities of a pretrained large Vision-Language Model (VLM). Specifically, we prompt GPT-5~\cite{openai2025gpt5systemcard} to evaluate generated videos along four dimensions: (1) Video Quality (VQ) measures the quality of generated video, independent of the source video and editing instructions; (2) Instruction Adherence (IA) measures the instruction-following ability of the editing prompt; (3) Source Video Preservation (SP) measures the preservation of the source video in unedited areas; (4) an overall score (OVL), for general assessment.

\emph{Human Study}.\ We further conduct a user study involving 10 participants, who independently rate each method along the same three dimensions: (1) Video Quality (VQ);  (2) Instruction Adherence (IA); and (3) Source preservation (SP). We averaged the three scores to obtain an overall human evaluation result for each video.
We randomly select one editing instruction for each video in our evaluation set to construct a subset for human evaluation. Participants are required to give a 1 to 5 rate on each criteria for each generated video. More details of evaluation metrics are shown in Supplementary~\cref{app:eval}.

\subsection{Comparison with Previous Methods}
\begin{wraptable}[11]{R}{0.5\textwidth}
 \vspace{-30pt}
  \caption{\textbf{Human evaluation results}.Best results are highlighted in \textcolor{red}{red}, and second-best in \textcolor{blue}{blue}.
  }
  \label{tab:video_editing_comparison—human}
  \centering
  \renewcommand{\arraystretch}{1}
  \setlength{\tabcolsep}{3pt}
  \begin{tabular}{lcccc}
    \toprule
    Methods & VQ $\uparrow$& IA $\uparrow$& SP $\uparrow$&Avg.$\uparrow$  \\
    \midrule
      
     AnyV2V~\cite{ku2024anyvv}  &2.50&2.61&2.48&2.54\\
     Ground-A-Video~\cite{jeong2024groundavideo} &1.95& 2.15 &2.37&2.16\\
     Token-flow~\cite{geyer2024tokenflow} &2.83& 2.63&3.16&2.87 \\

      Uniedit~\cite{bai2025uniedit} &2.36& 2.94&2.08 &2.46\\
         Ditto~\cite{bai2025ditto} &\textcolor{blue}{3.85}& \textcolor{blue}{3.20} &3.07 &3.37\\
      Flow-align~\cite{kim2026flowalign} &3.73&2.82&\textcolor{red}{3.81} &\textcolor{blue}{3.45}\\
     \rowcolor{yellow!30} ElasticTTT&\textcolor{red}{3.87}&\textcolor{red}{3.32}&\textcolor{blue}{3.61}&\textcolor{red}{3.60}  \\
    % p-score & 0.95& 0.81 &0.97& 0.625 \\
    
  \bottomrule
  \end{tabular}
\end{wraptable}

\cref{tab:video_editing_comparison} summarizes the quantitative comparison between ElasticTTT and competitive baseline methods. In general, TTT-based methods outperform prevailing video editing approaches due to their ability to adapt to the specific content and temporal dynamics of each input video at test time. Moreover, ElasticTTT achieves the strongest performance among TTT-based methods, with particularly substantial gains on VLM-based evaluations, which provide a holistic and fine-grained assessment of video editing quality across multiple dimensions (VQ 6.19, IA 7.50, and OVL 6.68). Although certain baselines match or exceed ElasticTTT on isolated metrics, they suffer from significant trade-offs. 

For example, Flow-align excels at preservation (SP 7.23) in unedited areas but struggles with editing elasticity and instruction adherence (IA 5.44). Conversely, Uniedit achieves strong semantic alignment with target instructions (evidenced by a high CLIP-T score of 31.76) but fails in preservation (SP 0.54), heavily distorting the original video. As shown in~\cref{fig:qualitative_comparison}, ElasticTTT achieves both precise editing adherence (ship and river, black stones) and source preservation (background building, rhino and trees), whereas other methods exhibit obvious compromises.

As shown in~\cref{tab:video_editing_comparison—human}, the results of human evaluation provide strong support for our quantitative findings. ElasticTTT achieves the highest performance (Average 3.60), indicating a clear human preference over the baseline methods. It excels particularly in editing adherence (IA 3.32) and achieves the top score for video quality (VQ 3.87). Even in preservation, where Flow-align peaks, ElasticTTT maintains a strong second-best performance (SP 3.61) without the severe trade-offs observed in competing models. Additionally, a study provided in the Supplementary~\cref{human} reveals a strong alignment between human volunteer scores and VLM evaluations, demonstrating the rationality of our VLM-based evaluation protocol.

\begin{wraptable}{r}[-1.5em]{0.5\textwidth}
  \centering
   % 视情况调整，消除上方空白
  \renewcommand{\arraystretch}{1.2}
  \setlength{\tabcolsep}{5pt}
 % \vspace{200pt}
  \caption{\textbf{Generalization to a larger base model.}Best results are highlighted in \textcolor{red}{red}.}
  \label{tab:model_scaling}
  \begin{tabular}{lcccc}
    \toprule
    Methods   & VQ$\uparrow$ & IA$\uparrow$ & SP$\uparrow$ & OVL$\uparrow$\\
    \midrule
    Wan2.2-5B-baseTTT & 5.05 & 5.09 & 4.40 & 4.91\\
    Wan2.2-5B-ours & \textcolor{red}{6.47} & \textcolor{red}{7.75} & 3.68 & \textcolor{red}{7.00}\\
    Wan2.1-1.3B-baseTTT & 5.34 & 6.15 & 3.63 & 5.39\\
    Wan2.1-1.3B-ours & {6.19} & {7.50} & \textcolor{red}{5.23} & 6.68\\
    \bottomrule
  \end{tabular}
   % 视情况调整，消除下方空白
\end{wraptable}

To examine whether ElasticTTT generalizes to larger base models, we further conduct the same experiment on Wan2.2-5B, comparing our full method against the vanilla TTT baseline under identical settings. As shown in \cref{tab:model_scaling}, ElasticTTT yields consistent and even more pronounced improvements on the larger backbone: it raises the overall score from 4.91 to 7.00 (+2.09), surpassing the corresponding gain observed on Wan2.1-1.3B (+1.29, from 5.39 to 6.68). The improvement is most evident in instruction adherence (IA 5.09 $\rightarrow$ 7.75) and video quality (VQ 5.05 $\rightarrow$ 6.47), indicating that our target distribution regularization and Contrastive CFG effectively unlock the richer generative prior of the larger model, which would otherwise be suppressed by TTT memorization. Notably, Wan2.2-5B equipped with ElasticTTT achieves the best overall score among all configurations (OVL 7.00), suggesting that our method scales favorably with model capacity. We observe a mild decrease in source preservation (SP 4.40 $\rightarrow$ 3.68), which we attribute to the stronger editing elasticity restored in the larger model: it follows editing instructions more aggressively, trading marginal preservation for substantially better semantic alignment—a trade-off that clearly pays off in the overall quality.

\subsection{Ablation Study}

\begin{figure}[t]
    \centering
    \includegraphics[width=1.0\linewidth]{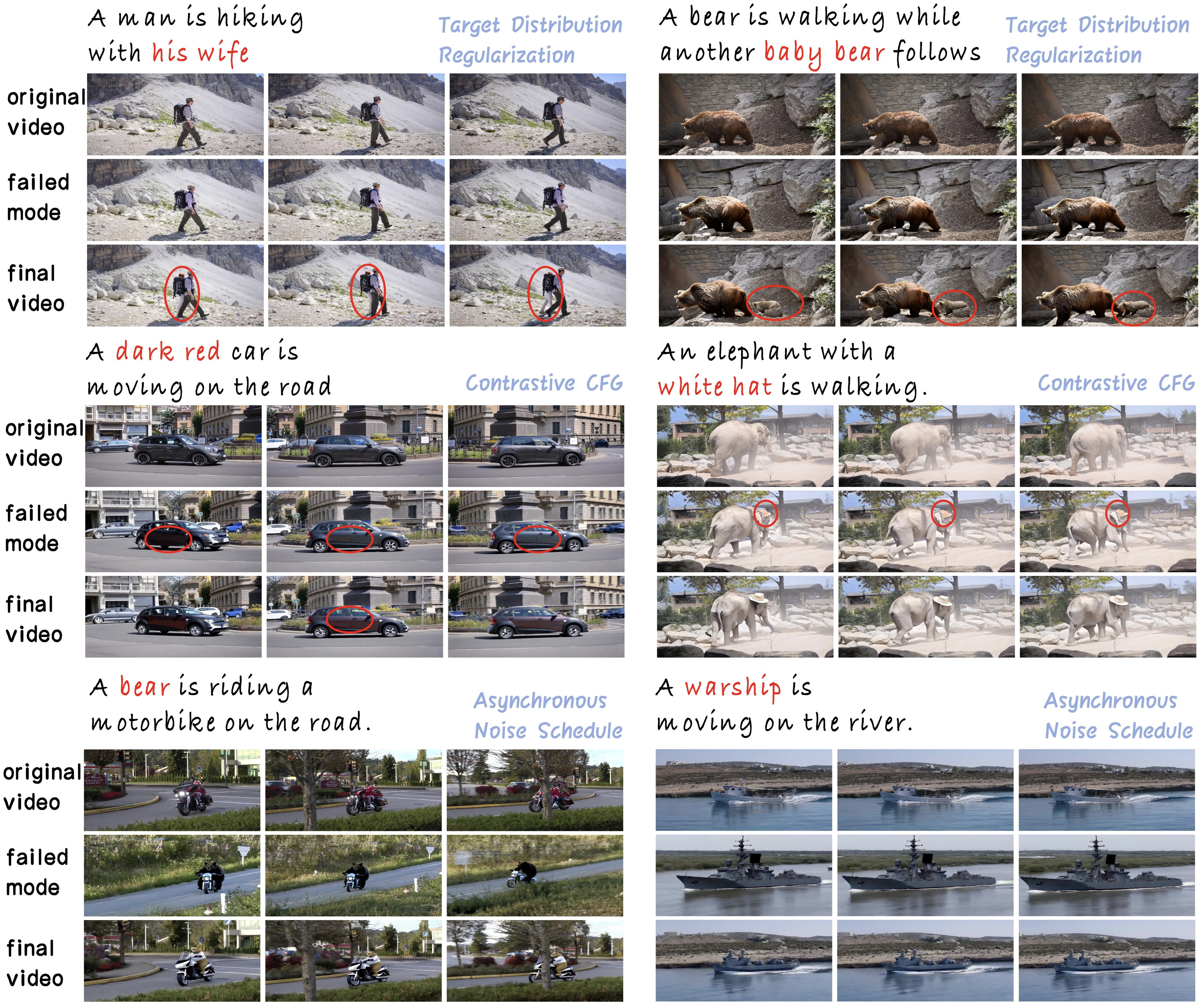}
    \caption{Visualized results of the ablation study.}
    \label{fig:ablation study}
    \vspace{-5mm}
\end{figure}
\begin{wraptable}{r}[-1.5em]{0.5\textwidth}
    \vspace{-10mm}
  \caption{\textbf{Quantitative results of the ablation study. }We exclude key components separately from the full model and report the resulting performance.Best results are highlighted in \textcolor{red}{red}. }
  
  \centering
  %   \scriptsize
  \renewcommand{\arraystretch}{1}
  \setlength{\tabcolsep}{5pt}
  %\vspace{-15pt}

  \label{tab:ablation}
  \begin{tabular}{lccccccc}
    \toprule
    Methods   & VQ$\uparrow$&IA$\uparrow$&SP$\uparrow$&OVL$\uparrow$\\
    \midrule
    None (baseline) &5.34&6.15&3.63&5.39\\
     w/o smooth transition &6.07&7.37&5.09&6.56\\
     w/o Async-NS&5.62&6.55&3.50&6.23\\
   
    w/o TDR  &6.13&7.45&5.36&6.64\\
    w/o contrastive CFG  &6.06&7.01&\textcolor{red}{5.64}&6.52\\
    \rowcolor{yellow!30}Full (ElasticTTT) &\textcolor{red}{6.19}&\textcolor{red}{7.50}&5.23&\textcolor{red}{6.68} \\
    
  \bottomrule
  \end{tabular}
\vspace{3mm}
\end{wraptable}

To assess individual contributions of our proposed designs, we conducted an ablation study systematically removing key components from the ElasticTTT pipeline. The results, presented both quantitatively in~\cref{tab:ablation} and qualitatively in~\cref{fig:ablation study}, confirm that the full method achieves the optimal balance across metrics, and also exhibits the best performance in~\cref{fig:ablation study}.

Omitting TDR negatively affects video quality, editing adherence, and overall score. Although removing it marginally improves the preservation score, this comes at the expense of the model's ability to execute complex edits. As shown in the first row in~\cref{fig:ablation study}, without TDR, the model merely reconstructs the source video and fails to add new concepts (wife, baby bear).

Removing Contrastive CFG leads to a noticeable decline in editing adherence and overall performance. 
Without this mechanism, the model struggles to fully follow the editing prompt, often resulting in an under-edited intermediate state between the reference and the target. As demonstrated in~\cref{fig:ablation study} (second row), the absence of Contrastive CFG results in subtle, less distinguishable edits that lack the visual clarity of our full method.

Both Async-NS and the additional smooth transitions are critical for cohesive video generation. Removing Async-NS causes a drop in all quantitative metrics. Qualitatively, as seen in the third row of~\cref{fig:ablation study}, without Async-NS the model loses its ability to accurately isolate the target regions, incorrectly modifying the foreground and background simultaneously. Meanwhile, using smooth transition map $\tilde{M}$ instead of $M$ is verified to be beneficial, as it results in a minor but distinct improvement in the overall score.
\begin{wraptable}{r}[-1.5em]{0.45\textwidth}
  \centering
  \vspace{-10pt}
  \caption{\textbf{Comparison between Async-NS and re-implemented Latent Blending under identical settings.} Best results are highlighted in \textcolor{red}{red}.}
  \label{tab:latent_blending}
  \begin{tabular}{lcccc}
    \toprule
    Methods   & VQ$\uparrow$ & IA$\uparrow$ & SP$\uparrow$ & OVL$\uparrow$\\
    \midrule
    Latent Blending & 4.65 & 6.02 & 1.08 & 4.71\\
    Async-NS & \textcolor{red}{6.19} & \textcolor{red}{7.50} & \textcolor{red}{5.23} & \textcolor{red}{6.68}\\
    \bottomrule
  \end{tabular}
  \vspace{-10pt}
\end{wraptable}
% \begin{bluehl}
We attribute this significant gap (OVL 6.68 vs. 4.71) to two primary factors.
First, Async-NS introduces \emph{asynchronous noise} inside the DiT forward process. By assigning different noise levels and timestep embeddings to the edited and preserved regions, it forms distinct regional representations during velocity prediction, thereby mitigating TTT-induced spatial entanglement. In contrast, post-hoc blending approaches only replace or mix latents after prediction, and thus cannot prevent entangled foreground/background features from being formed in the first place.

Second, such approaches rigidly hard-preserve the background, which proves counterproductive for realistic video editing: the unmasked region often requires soft physical adjustments, such as shadows, dust, boundary deformation, and contact effects. Forcibly stitching frozen background latents onto the edited content at mismatched noise levels produces visible seams and temporal artifacts, which explains the catastrophic preservation score of Latent Blending (SP 1.08) despite its seemingly conservative design. Async-NS instead suppresses spurious background drift while permitting these necessary interactions—a forward-pass desynchronization mechanism rather than a post-forward background replacement.

% \end{bluehl}

\begin{comment}
    
\emph{Async-NS} We conduct a two-stage ablation study. First, omitting the smooth transition introduces abrupt noise variations, failing to seamlessly connect edited and preserved areas, and degrading all metrics (Table \ref{tab:ablation}). Second, completely removing the strategy (using uniform initial noise) yields a better prompt alignment but destroys the structural details of the preserved regions, ultimately lowering the VLM and editing scores.

\emph{Contrastive CFG} 
Table \ref{tab:ablation} presents the ablation study in contrast Classifier-Free Guidance (CFG). Without it, the model struggles to fully follow the editing prompt, often resulting in an under-edited intermediate state between the reference and the target. Conversely, applying contrastive CFG enables more flexible editing, significantly improving semantic alignment and overall editing quality.

\emph{TDR} 
To validate the TDR, we conducted an ablation study (Table \ref{tab:ablation}). The results demonstrate that adding noise improves nearly all aspects of the generated video. Furthermore, applying this regularization strategy simultaneously improves both editing flexibility and overall visual quality.

\end{comment}

\section{Conclusion}

In this paper, we present \textbf{ElasticTTT}, a novel Test-Time Tuning based video editing framework. We first analyze the prior collapse phenomenon brought by the inherent tension between the rigid optimization target in TTT and the generative nature of diffusion models. By introducing three novel techniques: Target Distribution Regularization, Contrastive CFG, and Asynchronous Noise Scheduling, we resolve the Conditioning Collapse and Spatial Entanglement issues caused by prior collapse and achieve SOTA performance on quantitative results and human evaluation. Detailed ablation studies quantify the contribution of each key component in our framework. We believe our work not only represents an important step toward more reliable diffusion-based models, but also highlights promising directions for designing better personalized AI systems.

%%%%%%%%%%%%%%%%%%%%%%%%%%%%%%%%%%%%%%%%%%%%%%%%%%%%%%%%%%%%

\bibliographystyle{abbrv}
\bibliography{main}

\newpage
\appendix

This is the supplementary material for the paper ``ElasticTTT:Prior-Preserving Test-Time Tuning for Video Editing''. We organize the content as follows:

\vspace{1em}

\noindent \textbf{\textcolor{red}{\ref{math} --}} \textbf{Mathematical Derivations} \\[0.5em]
\noindent \textbf{\textcolor{red}{\ref{exp} --}} \textbf{Details on Experiment Setup} \\[0.5em]
\noindent \textbf{\textcolor{red}{\ref{eval} --}} \textbf{Details on Evaluation Metrics} \\[0.5em]
\noindent \textbf{\textcolor{red}{\ref{human} --}} \textbf{Details on Human Evaluation} \\[0.5em]
\noindent \textbf{\textcolor{red}{\ref{more} --}} \textbf{More Evaluation Results} \\[0.5em]
\noindent \textbf{\textcolor{red}{\ref{corr} --}} \textbf{Correlation Analysis of VLM and Human Judgments} \\[0.5em]
\noindent \textbf{\textcolor{red}{\ref{prompt} --}} \textbf{Prompt Design for VLM Judgment} \\[0.5em]
\noindent \textbf{\textcolor{red}{\ref{async} --}} \textbf{Demonstration of Asynchronous Noise Scheduling Masks} \\[0.5em]
\noindent \textbf{\textcolor{red}{\ref{limit} --}} \textbf{Limitations and Future Work} \\[0.5em]
\noindent \textbf{\textcolor{red}{\ref{demo} --}} \textbf{More Demonstrations} \\[0.5em]
\noindent \textbf{\textcolor{red}{\ref{social} --}} \textbf{Social Impact and Ethical Concerns} \\[0.5em]

\section{Mathematical Derivations}
\label{math}
\label{app:math}
\subsection{Target Distribution Regularization}
We demonstrate in Sec. 3.2 that Target Distribution Regularization (TDR) does not alter the optimization target of the diffusion model. However, it cannot fully explain why TDR preserves the general structure of the generated distribution. Here, we offer an analysis from the perspective of optimization dynamics. Specifically, we analyze the covariance of the stochastic gradients to demonstrate how standard TTT is prone to degenerate memorization, and how our proposed TDR implicitly injects a geometry-aware Gauss-Newton penalty that preserves the pre-trained generative prior.

\subsubsection{Setup and Standard TTT}

Let the pre-trained network's prediction be denoted as $v_\theta(x_t, t, C) \in \mathbb{R}^d$, parameterized by $\theta \in \mathbb{R}^p$, where $x_t$ is the noisy latent, $t$ is the timestep, and $C$ represents the text conditioning. Let the true diffusion target be $u_t$. During standard TTT, the objective for a single stochastic step is the $L_2$ loss:
\begin{align}
\mathcal{L}(\theta) = \frac{1}{2} \| v_\theta(x_t, t, C) - u_t \|^2
\end{align}

Let the instantaneous error vector be $e_t(\theta) = v_\theta(x_t, t, C) - u_t \in \mathbb{R}^d$. We denote the Jacobian of the network output with respect to its parameters as $J_\theta = \nabla_\theta v_\theta(x_t, t, C) \in \mathbb{R}^{d \times p}$.

\begin{definition}[Stochastic Gradient and Covariance]
The stochastic gradient for a single iteration of standard TTT is $g(\theta) = \nabla_\theta \mathcal{L} = J_\theta^T e_t(\theta)$. The expected (true) gradient is $\bar{g} = \mathbb{E}_{t, \epsilon}[g(\theta)]$. The covariance matrix of this stochastic gradient is defined as:
\begin{align}
\Sigma(\theta) = \mathbb{E}_{t, \epsilon} \left[ (g(\theta) - \bar{g})(g(\theta) - \bar{g})^T \right]
\end{align}
\end{definition}

\begin{proposition}[Prior Collapse of Standard TTT]
Assume the highly overparameterized model is capable of perfect memorization of a target $x_0$. As the model perfectly overfits the spatial mapping, the instantaneous error vanishes ($e_t(\theta) \to \mathbf{0}$), causing the gradient covariance to vanish: $\Sigma(\theta) \to \mathbf{0}$.
\end{proposition}

\begin{proof}
By definition, $g(\theta) = J_\theta^T e_t(\theta)$. As the model perfectly fits the target distribution, $v_\theta(x_t, t, C) \to u_t$, yielding $e_t(\theta) \to \mathbf{0}$. Consequently, for all $t$ and $\epsilon$, $g(\theta) \to \mathbf{0} \Rightarrow \bar{g} \to \mathbf{0} \Rightarrow \Sigma(\theta) \to \mathbf{0}$.
\end{proof}

In the Stochastic Differential Equation (SDE) view of Stochastic Gradient Descent ($d\theta = -\eta \bar{g} dt + \sqrt{\eta \Sigma} dW$)~\cite{mandt2017stochastic}, the term $\Sigma$ governs the exploration of the optimizer. When $\Sigma \to \mathbf{0}$, the optimizer loses all explorative ability and becomes irreversibly trapped in a sharp singularity where semantic conditioning $C$ and prior $x_t$ are bypassed in favor of rigid memorization.

\subsubsection{Target Distribution Regularization (TDR)}

To prevent prior collapse, we introduce a zero-mean target regularization term $\epsilon_{\text{reg}} \sim \mathcal{N}(0, \sigma_{\text{reg}}^2 I)$. The regularized target becomes $\tilde{u}_t = u_t + \epsilon_{\text{reg}}$, while the input $x_t$ remains strictly unmodified.

\begin{proposition}[Unbiased Expected Gradient]
The expected gradient under Target Distribution Regularization (TDR) remains strictly identical to the expected gradient of Standard TTT: $\bar{g}_\text{TDR} = \bar{g}$.
\end{proposition}

\begin{proof}
The regularized stochastic gradient is given by:
\begin{align}
g_\text{TDR}(\theta) = J_\theta^T (v_\theta(x_t, t, C) - (u_t + \epsilon_{\text{reg}})) = g(\theta) - J_\theta^T \epsilon_{\text{reg}}
\end{align}
Taking the expectation over $t, \epsilon$, and the independent noise $\epsilon_{\text{reg}}$:
\begin{align}
\mathbb{E}[g_\text{TDR}(\theta)] = \mathbb{E}_{t, \epsilon} [g(\theta)] - \mathbb{E}_{t, \epsilon} \left[ J_\theta^T \right] \mathbb{E}_{\epsilon_{\text{reg}}} [\epsilon_{\text{reg}}]
\end{align}
Since $\epsilon_{\text{reg}}$ is zero-mean, the second term vanishes, yielding $\bar{g}_\text{TDR} = \bar{g}$.
\end{proof}

\begin{proposition}[Geometry-Aware Covariance Injection]
Under TDR, the gradient covariance matrix analytically decomposes into the standard covariance plus a Gauss-Newton penalty term proportional to $\sigma_{\text{reg}}^2$:
\begin{align}
\Sigma_\text{TDR}(\theta) = \Sigma(\theta) + \sigma_{\text{reg}}^2 \mathbb{E}_{t, \epsilon} \left[ J_\theta^T J_\theta \right]
\end{align}
\end{proposition}

\begin{proof}
By definition, the new covariance is:
\begin{align}
\Sigma_\text{TDR}(\theta) = \mathbb{E} \left[ (g_\text{TDR}(\theta) - \bar{g})(g_\text{TDR}(\theta) - \bar{g})^T \right]
\end{align}
Substitute $g_\text{TDR}(\theta) = g(\theta) - J_\theta^T \epsilon_{\text{reg}}$ into the expectation:
\begin{align}
\Sigma_\text{TDR}(\theta) &= \mathbb{E} \left[ \left( (g - \bar{g}) - J_\theta^T \epsilon_{\text{reg}} \right) \left( (g - \bar{g}) - J_\theta^T \epsilon_{\text{reg}} \right)^T \right] \nonumber \\
&= \mathbb{E} \left[ (g - \bar{g})(g - \bar{g})^T \right] - \mathbb{E} \left[ (g - \bar{g}) \epsilon_{\text{reg}}^T J_\theta \right] \nonumber \\
&\quad - \mathbb{E} \left[ J_\theta^T \epsilon_{\text{reg}} (g - \bar{g})^T \right] + \mathbb{E} \left[ J_\theta^T \epsilon_{\text{reg}} \epsilon_{\text{reg}}^T J_\theta \right]
\end{align}
Because $\epsilon_{\text{reg}}$ is zero-mean and independent of the standard gradient components $g$ and $J_\theta$, the expectations of the two cross-terms evaluate strictly to zero:
\begin{align}
\mathbb{E} \left[ (g - \bar{g}) \epsilon_{\text{reg}}^T J_\theta \right] = \mathbb{E}_{t,\epsilon} \left[ (g - \bar{g}) \right] \mathbb{E}_{\epsilon_{\text{reg}}} [\epsilon_{\text{reg}}^T] \mathbb{E}_{t,\epsilon}[J_\theta] = \mathbf{0}
\end{align}
Thus, the total expectation simplifies to the sum of the expectations of the two remaining terms. Recognizing that the first term is exactly the definition of the standard covariance $\Sigma(\theta)$:
\begin{align}
\Sigma_\text{TDR}(\theta) = \Sigma(\theta) + \mathbb{E}_{t, \epsilon} \left[ J_\theta^T \mathbb{E}_{\epsilon_{\text{reg}}}[\epsilon_{\text{reg}} \epsilon_{\text{reg}}^T] J_\theta \right]
\end{align}
Since the covariance of the injected noise is $\mathbb{E}[\epsilon_{\text{reg}} \epsilon_{\text{reg}}^T] = \sigma_{\text{reg}}^2 I$, the expression rigorously evaluates to:
\begin{align}
\label{sigmatdr}
\Sigma_\text{TDR}(\theta) = \Sigma(\theta) + \sigma_{\text{reg}}^2 \mathbb{E}_{t, \epsilon} \left[ J_\theta^T J_\theta \right]
\end{align}
\end{proof}

\begin{proposition}[Implicit Jacobian Penalty]
TDR injects optimizer noise proportional to the squared Frobenius norm of the Jacobian, mathematically measuring the ``sharpness''~\cite{foret2020sharpness, hochreiter1997flat, keskar2016large} of the parameter space:
\begin{align}
\text{Tr}(\Sigma_\text{TDR}) = \text{Tr}(\Sigma) + \sigma_{\text{reg}}^2 \mathbb{E}_{t, \epsilon} \left[ \| J_\theta \|_F^2 \right]
\end{align}
\end{proposition}
\begin{proof}
Taking the Trace over both sides of~\cref{sigmatdr} yields.
\end{proof}

The term $\mathbb{E}[J_\theta^T J_\theta]$ is the uncentered covariance of the Jacobian, commonly recognized as the Gauss-Newton approximation~\cite{gauss1877theoria} of the Hessian under $L_2$ losses. If the network attempts to form a degenerate, stiff spatial singularity (bypassing conditions and priors to perfectly memorize $x_0$), the Jacobian norm $\| J_\theta \|_F^2$ explodes. 

Consequently, even as the empirical loss approaches zero ($e_t \to 0$ and $\Sigma \to \mathbf{0}$), the optimizer is subjected to persistent, geometry-aware noise: $\sigma_{\text{reg}}^2 \| J_\theta \|_F^2$. This exploding covariance violently ejects the optimizer from sharp minima via SDE Brownian motion. To converge, the network is physically forced to reach a parameter setting where $\| J_\theta \|_F^2$ remains small and stable. The original pre-trained generative prior, where the predicted velocity matches the ground truth across the entire noisy distribution, represents precisely this required flat configuration. Thus, TDR forbids prior collapse not by altering the global expected gradient, but by rendering collapsed states physically unreachable.

\subsection{Contrastive Classifier-Free Guidance}

We claim in Sec. 3.3 that our Contrastive Classifier-Free Guidance (CFG)~\cite{ho2022classifier} actively repels the generated distribution from the source distribution, overcoming the distributional shift brought by the single-data optimization. 
This section provides a formal Bayesian analysis of this mechanism, providing a theoretical perspective. We formulate this mechanism as a dynamic, geometry-aware filter that mathematically neutralizes the distributional shift induced by TTT.

\subsubsection{The TTT Gravity Well}

To elucidate the failure of standard CFG post-TTT, we analyze the optimization dynamics through the lens of Energy-Based Models (EBMs)~\cite{lecun2006tutorial,song2020score, song2019generative}. In diffusion frameworks, the network estimates the score function~\cite{song2019generative} of a data distribution: $s_\theta(x_t, C) = \nabla_{x_t} \log p_\theta(x_t | C)$.

\begin{definition}[The Energy Landscape of TTT]
Let the pretrained base model describe an implicit probability distribution $p_{\text{base}}(x_t | C)$. TTT aggressively optimizes the network to reconstruct the source video trajectory $x_{\text{src}}$, effectively forcing the model to assign maximum likelihood to $x_{\text{src}}$ when conditioned on the source prompt $C_{\text{src}}$.
\end{definition}

\begin{proposition}[The Gradient Bias Field]
During prior collapse, TTT structurally shifts the base distribution by introducing a highly localized memorization peak. After TTT, the generated distribution shifts, and the score of the shifted distribution decomposes into the base score and an additive gradient bias field, $\nabla_{x_t} B(x_t)$.
\end{proposition}

\begin{proof}
Because neural networks are universal function approximators, the overfitted probability distribution $p_{\theta^*}$ can be factored into the original pre-trained prior and a new, TTT-induced memorization peak:
\begin{align}
p_{\theta^*}(x_t | C_{\text{src}}) = \frac{1}{Z} p_{\text{base}}(x_t | C_{\text{src}}) \cdot p_{\text{mem}}(x_t)
\end{align}
where $p_{\text{mem}}(x_t)$ approaches a Dirac delta function $\delta(x_t - x_{\text{src}})$ as the TTT loss converges, and $Z$ is a regularization term to guarantee $\int p_{\theta^*}(x_t|C_\text{src})\mathrm{d}x_t=1$. The tuned score function is the log-likelihood gradient of this factorization:
\begin{align}
s_{\theta^*}(x_t, C_{\text{src}}) = \nabla_{x_t} \log p_{\text{base}}(x_t | C_{\text{src}}) + \nabla_{x_t} \log p_{\text{mem}}(x_t)
\end{align}
Defining the scalar field $B(x_t) = \log p_{\text{mem}}(x_t)$ yields the additive form:
\begin{align}
s_{\theta^*}(x_t, C_{\text{src}}) = s_{\text{base}}(x_t, C_{\text{src}}) + \nabla_{x_t} B(x_t)
\end{align}
\end{proof}

Mathematically, $B(x_t)$ acts as a deep, artificial energy well sculpted into the latent space. Its gradient, $\nabla_{x_t} B(x_t)$, acts as a spatial force vector aggressively pulling latents toward the source video $x_{\text{src}}$. We name this gradient bias $\nabla_{x_t} B(x_t)$ the \textit{source attractor bias}.

\subsubsection{Semantic Leakage of the TTT Bias}

Due to parameter sharing across conditions, this localized energy well unavoidably ``bleeds'' into other prompt conditions during inference.

\begin{definition}[Condition-Dependent TTT Bias]
We define the tuned score function for an arbitrary prompt $C$ as the ideal base prior altered by a condition-dependent activation of the source attractor bias:
\begin{align}
\label{eq:bias}
s_{\theta^*}(x_t, C) = s_{\text{base}}(x_t, C) + \alpha(C) \nabla_{x_t} B(x_t)
\end{align}
where $\alpha(C) \in [0, 1]$ parameterizes the activation strength.
\end{definition}

\begin{proposition}[Semantic Proportionality of Activation]
The activation coefficient $\alpha(C)$ is directly correlated to the semantic similarity between the query prompt $C$ and the tuned source prompt $C_{\text{src}}$ within the text embedding space.
\end{proposition}

\noindent \textbf{Justification.} During TTT, conditioning projections are updated specifically for the embedding of $C_{\text{src}}$. For a novel prompt $C$, the triggering of these updated weights depends on embedding alignment, well-approximated by cosine similarity: $\alpha(C) \approx \text{sim}(C, C_{\text{src}})$. This establishes three distinct activation regimes:
\textbf{Source Prompt ($C_{\text{src}}$):} Similarity is maximized ($\text{sim} = 1$). Thus $\alpha(C_{\text{src}}) \approx 1$, fully activating the gravity well.\\
\textbf{Target Prompt ($C_{\text{trg}}$):} Target prompts share significant semantic overlap with the source (e.g., shared subjects). Hence, $0 < \alpha(C_{\text{trg}}) < 1$. This strict positivity inadvertently triggers a partial collapse into the TTT gravity well.\\
\textbf{Negative Prompt ($C_{\text{neg}}$):} Generic degradation prompts are semantically orthogonal to the source. Thus $\alpha(C_{\text{neg}}) \approx 0$, successfully bypassing the overfitted parameters.

\subsubsection{Active Cancellation via Contrastive CFG}

To escape this gravity well and overcome the distributional shift, our Contrastive CFG defines the modified score as:
\begin{align}
\label{eq:dual_cfg}
s_{\text{ContraCFG}}(x_t) &= s_{\theta^*}(x_t, C_{\text{neg}}) + \lambda_1 \left( s_{\theta^*}(x_t, C_{\text{trg}}) - s_{\theta^*}(x_t, C_{\text{neg}}) \right) \nonumber\\
&\quad+ \lambda_2 \left( s_{\theta^*}(x_t, C_{\text{trg}}) - s_{\theta^*}(x_t, C_{\text{src}}) \right)
\end{align}

\begin{proposition}[Standard CFG Amplifies TTT Bias]
Standard Classifier-Free Guidance applied post-TTT inherently amplifies the source memorization, restricting semantic editing.
\end{proposition}

\begin{proof}
Substituting~\cref{eq:bias} into the standard CFG formula yields:
\begin{align}
s_{\text{StdCFG}}(x_t) 
% Step 1: The standard CFG formula
&= s_{\theta^*}(x_t, C_{\text{neg}}) + \lambda_1 \left( s_{\theta^*}(x_t, C_{\text{trg}}) - s_{\theta^*}(x_t, C_{\text{neg}}) \right) \nonumber\\
% Step 3: Group the ideal base scores together and factor out \nabla B
&= \underbrace{ \Big\{ s_{\text{base}}(x_t, C_{\text{neg}}) + \lambda_1 \big( s_{\text{base}}(x_t, C_{\text{trg}}) - s_{\text{base}}(x_t, C_{\text{neg}}) \big) \Big\} }_{ s_{\text{base, StdCFG}}(x_t) }  \nonumber\\
&\quad + \Big\{ \alpha(C_{\text{neg}}) + \lambda_1 \big( \alpha(C_{\text{trg}}) - \alpha(C_{\text{neg}}) \big) \Big\} \nabla_{x_t} B(x_t)  \nonumber\\
% Step 4: Apply the fact that the negative prompt does not trigger the TTT bias (alpha(C_neg) = 0)
&\approx s_{\text{base, StdCFG}}(x_t) + \Big\{ 0 + \lambda_1 \big( \alpha(C_{\text{trg}}) - 0 \big) \Big\} \nabla_{x_t} B(x_t)  \nonumber\\
% Step 5: Final simplified result showing the amplified bias
&= s_{\text{base, StdCFG}}(x_t) + \lambda_1 \alpha(C_{\text{trg}}) \nabla_{x_t} B(x_t)
\end{align}

Since $\alpha(C_{\text{neg}}) \approx 0$, the bias simplifies to $\lambda_1 \alpha(C_{\text{trg}}) \nabla_{x_t} B(x_t)$. Because $\lambda_1 \gg 1$, standard CFG aggressively scales the residual bias, misleading the generative trajectory toward the direction of the source attractor bias.
\end{proof}

\begin{proposition}[Contrastive Cancellation of TTT Bias]
The contrastive guidance term injects a precisely scaled negative gradient that algebraically cancels the amplified TTT bias.
\end{proposition}

\begin{proof}
Evaluating the contrastive guidance term (scaled by $\lambda_2$):
\begin{align}
s_{\theta^*}(x_t, C_{\text{trg}}) - s_{\theta^*}(x_t, C_{\text{src}}) &= \left( s_{\text{base}}(x_t, C_{\text{trg}}) - s_{\text{base}}(x_t, C_{\text{src}}) \right) \nonumber \\
&\quad + \left( \alpha(C_{\text{trg}}) - \alpha(C_{\text{src}}) \right) \nabla_{x_t} B(x_t)
\end{align}
Because TTT strictly overfits $C_{\text{src}}$, we have $\alpha(C_{\text{src}}) > \alpha(C_{\text{trg}})$. Consequently, the bias scale is strictly negative. Let $\beta = \alpha(C_{\text{src}}) - \alpha(C_{\text{trg}}) > 0$, yielding:
\begin{align}
s_{\theta^*}(x_t, C_{\text{trg}}) - s_{\theta^*}(x_t, C_{\text{src}}) = \left( s_{\text{base}}(x_t, C_{\text{trg}}) - s_{\text{base}}(x_t, C_{\text{src}}) \right) - \beta \nabla_{x_t} B(x_t)
\end{align}
Substituting this directly back into the Contrastive CFG equation (~\cref{eq:dual_cfg}):
\begin{align}
s_{\text{ContraCFG}}(x_t) = s_{\text{base, ContraCFG}}(x_t) + \left[ \lambda_1 \alpha(C_{\text{trg}}) - \lambda_2 \beta \right] \nabla_{x_t} B(x_t)
\end{align}
\end{proof}

While standard CFG burdens generation with a heavily amplified source-attractor bias, the contrastive term explicitly isolates the geometry of the TTT gravity well and subtracts it with magnitude $\lambda_2 \beta$. Through appropriate scaling of $\lambda_2$, the net bias coefficient $[\lambda_1 \alpha(C_{\text{trg}}) - \lambda_2 \beta]$ evaluates to zero, achieving exact cancellation of the TTT distributional shift, fully restoring semantic editing capabilities without degrading the structural fidelity

\section{Details on Experiment Setup}
\label{exp}
\subsection{Implementation Details of Baseline Models}

In our main experiment, we evaluate our proposed method, ElasticTTT, against recent state-of-the-art video editing approaches. To comprehensively assess performance, we categorize our baselines into two groups: those utilizing their official open-source configurations and those we re-implemented to guarantee a standardized comparison.

For several previous models—specifically AnyV2V~\cite{ku2024anyvv}, Ground-A-Video~\cite{jeong2024groundavideo}, Token-Flow~\cite{geyer2024tokenflow}, Ditto~\cite{bai2025ditto}, and UniEdit~\cite{bai2025uniedit}—we directly obtain results using their official open-source code and model weights. We strictly adhere to the default hyperparameters provided in their official implementations, keeping the output video lengths and spatial resolutions at their original settings without modification. Regarding AnyV2V, the framework necessitates a dedicated image editing model to generate first-frame guidance; to maintain consistency with its original paper, we integrate InstructPix2Pix as the designated editor. However, a significant limitation of relying on these out-of-the-box methods is that their foundational settings, such as the chosen base models, vary drastically. This discrepancy makes direct comparisons between them and our method less conclusive.

To address this limitation and enable a truly fair comparison, we adopt select state-of-the-art training-free and test-time tuning editing methodologies and re-implement them on identical configurations to ElasticTTT. Specifically, we re-implement Flow-Align~\cite{kim2025flowalign} to serve as a highly competitive training-free baseline that shares our exact experimental settings. Flow-Align~\cite{kim2025flowalign} originally was a SOTA training-free image editing framework, where the methodology is model-independent and can be directly adapted to new architectures.

Although Tune-A-Video~\cite{wu2023tune} and VidTTA~\cite{wang2025low} were originally designed for 3D-UNet architectures, their core contributions are architecture-agnostic and can be readily adapted to Diffusion Transformer (DiT) frameworks. To ensure an equitable evaluation, we port their primary mechanisms into the Wan2.1-1.3B~\cite{wan2025wan} architecture. Furthermore, we align all key hyperparameters for these re-implemented models, such as the number of test-time tuning steps, learning rate, with those used in ElasticTTT to ensure fair comparison.

For all other evaluated baseline models, we strictly adhered to the default hyper-parameters provided in their official implementations.

\subsection{Implementation Details of ElasticTTT}

Regarding our training configuration, we align our primary test-time tuning phase with the protocol established by LoRA-Edit~\cite{gao2025lora}, specifically fixing the number of optimization steps to 100. We utilized the AdamW~\cite{kingma2014adam,loshchilov2017decoupled} optimizer with a learning rate set to 3e-5. We set the shift of the diffusion noise schedule to 3.0, and utilize 50 sampling steps during inference with an Euler sampler. All generated videos have 81 frames with 480p resolution, aligning with common protocol~\cite{wan2025wan}.

In our ablation study, the video samples demonstrating target distribution regularization were generated using 300 test-time tuning (TTT) steps, as we observed that this regularization mechanism performs more effectively with extended optimization. Conversely, the visual examples illustrating the contrastive CFG and AsyncNC ablation were obtained using 70 TTT steps. For the main visual comparisons against other state-of-the-art methods, all hyperparameters are kept identical to those used in our quantitative evaluation.

The methodology related hyperparameters introduced Sec.4.1 are determined empirically based on qualitative observations from a minimal validation set comprising only five video samples. Because we deliberately bypassed an exhaustive grid search or systematic hyperparameter optimization to save computational resources, the current configuration may be suboptimal. Consequently, we anticipate that our method possesses further headroom for performance improvement, which could be unlocked through a more rigorous hyperparameter search.

To further characterize how these hyperparameters affect performance, we evaluate the model across a diverse range of settings on our test set, as summarized in \cref{tab:stacked_result}, with the default configuration shown in \textcolor{red}{red}.

Overall, ElasticTTT exhibits strong robustness to the TDR and CFG scales. Varying the TDR scale from 0.1 to 0.3 changes the overall score by less than 0.05 (6.68--6.73), while any non-zero setting outperforms disabling TDR entirely (OVL 6.59), confirming that the benefit stems from the regularization itself rather than a delicately tuned variance. Similarly, all tested CFG scale combinations stay within a narrow band (6.63--6.75), indicating that the contrastive guidance is effective across a wide range of strengths. Notably, two non-default settings (TDR scale 0.3 with OVL 6.73, and CFG scale 6-3 with OVL 6.75) slightly surpass our default configuration, which corroborates the aforementioned headroom left by our deliberately coarse hyperparameter selection.

\begin{table}[htbp]
    \centering
  
    \caption{Overall scores (OVL) under different hyperparameter settings. The default configuration is shown in \textcolor{red}{red}.}
    \label{tab:stacked_result}
    \begin{tabular}{lcccc}
        \toprule
        \textbf{TDR scale} & 0 & 0.1 & \textcolor{red}{0.2} & 0.3 \\
        \textbf{OVL}       & 6.59 & 6.69 & \textcolor{red}{6.68} & 6.73 \\
        
        \midrule % 【关键修改 3】用单根中线替代原来的 \bottomrule + \addlinespace，大幅节省垂直空间
        
        \textbf{CFG scale} & 5-3 & 6-3 & \textcolor{red}{6-2} & 6-1 \\
        \textbf{OVL}       & 6.69 & 6.75 &\textcolor{red}{6.68} & 6.63 \\
        
        \midrule % 同上
        
        \textbf{Async-NS}  & 0.95-0.5 & 0.97-0.6 & \textcolor{red}{0.97-0.55} & 0.95-0.55 \\
        \textbf{OVL}       & 4.59 & 4.71 & \textcolor{red}{6.68} & 4.62 \\
        \bottomrule
    \end{tabular}
   % 【关键修改 4】增加表格下方的负间距，让正文更贴近表格
\end{table}

In contrast, the Async-NS noise regimes $(T_\text{e}, T_\text{p})$ constitute the most sensitive hyperparameter: deviating from the default $(0.97, 0.55)$ causes the overall score to drop sharply from 6.68 to around 4.6. This is expected, as these two thresholds directly define the desynchronization between the edited and preserved regions---an overly high $T_\text{p}$ injects excessive noise into the regions to be preserved and corrupts the source structure, whereas a lower $T_\text{e}$ provides insufficient noise for the edited regions to break away from the memorized source content.

\section{Details on Evaluation Metrics}
\label{eval}
\label{app:eval}

\begin{table}[!htbp]
  \caption{The elaborated results of six metrics on VBench~\cite{huang2024vbench}
  }
  \label{tab:Vbench detail score}
  \centering
    \setlength{\tabcolsep}{6pt} % increase column spacing
    \renewcommand{\arraystretch}{1.15} % increase row spacing
  \begin{tabular}{lcccccc}
    \toprule
    Methods &SC$\uparrow$ &BC$\uparrow$& MS$\uparrow$&DD$\uparrow$& AQ$\uparrow$&IQ$\uparrow$  \\
    \midrule
    \multicolumn{7}{l}{\textit{Other Methods}} \\     
     Flow-align & 0.900&0.910& 0.971&0.960&0.519&67.056\\
     Ditto &  0.927&0.924 &0.977&0.760& 0.574&71.862\\
      Uniedit & 0.957& 0.970&0.983&0.352&0.498&55.033\\
     AnyV2V & 0.850 &0.906 &0.937&0.928&0.526&65.215\\
     Ground-A-Video & 0.871 &0.918 &0.872&0.960&0.561&71.065\\
     Token-flow &0.936 &0.944 &0.969 &0.752&0.523&70.486\\
    \midrule
    \multicolumn{7}{l}{\textit{Test-Time-Tuning Methods}} \\ 
     Tune-A-Video&0.937&0.924&0.973&0.960&0.530&67.886\\
     VidTTA&0.931&0.917&0.972&0.960&0.526&67.900\\
     ElasticTTT&0.934& 0.920&0.971&0.969&0.534&66.554  \\
    
  \bottomrule
  \end{tabular}
\end{table}

\subsection{Automatic Metrics}

We provide a detailed introduction to the automatic evaluation protocol adopted.

\emph{CLIP-T}.\ To quantify the alignment between the generated visual content and the target text instructions, we measure the CLIP~\cite{radford2021learning} text-image similarity (CLIP-T). This metric leverages the pre-trained CLIP~\cite{radford2021learning} model to map both text and images into a shared latent space. For our evaluation, we uniformly sample 8 frames from each video sequence. The cross-modal similarity score is computed independently by calculating the cosine similarity between the embedding of each individual frame and the embedding of the edited prompt. The final reported CLIP-T value is the arithmetic mean of these 8 independent frame-level similarity scores.

\emph{VEBench}. \ VEBench~\cite{liu2025vebench} is a subjective-aligned benchmark suite specifically designed for text-driven video editing quality assessment. Traditional objective metrics often struggle to align with human perception; to overcome this, VEBench utilizes a dedicated multi-modal quality assessment network (VE-Bench QA). Rather than relying on generic vision-language models, this specialized network evaluates edited videos by simultaneously modeling three critical components: the inner connection and relevance between the source and edited videos (source-target relationship), the alignment between the edited video and the driving text prompt, and overall visual quality indicators such as aesthetics and distortion. To ensure strict reproducibility during our VEBench evaluation process, we fixed the random seed to 666.

\emph{VBench}. \ VBench~\cite{huang2024vbench} is a comprehensive evaluation suite that dissects video generative performance into 16 hierarchical, disentangled, and fine-grained dimensions (such as temporal flickering, motion smoothness, subject consistency, and aesthetic quality). This allows for a highly objective and human-aligned assessment of video generation capabilities. In our experiments, we select six core aspects that are closely related to the editing task, and are directly measurable on data other than the standard VBench testset.
Specifically, we calculate Subject Consistency (SC), Background Consistency (BC), Motion Smoothmess (MS), Dynamic Degree (DD), Aesthetic Quality (AQ), and Image Quality (IQ).
Because the scoring scales and computational methodologies across these disentangled dimensions inherently vary, we systematically normalize the raw values of these six aspects. These normalized values are then summed to compute a single, unified overall editing score. A detailed breakdown of the specific scores across all evaluated VBench dimensions is provided in~\cref{tab:Vbench detail score}.

\subsection{Judge-Based Metrics}
The above automatic metrics provide an overall rating of the generated videos. However, these metrics focus on the general information, and overlook fine-grained video details, as well as the subtle semantics in the editing prompt.
To provide a nuanced, reasoning-based assessment of our editing results, we employ a Vision-Language Model (VLM) as an automated evaluator. Specifically, we utilize GPT-5~\cite{openai2025gpt5systemcard} to perform complex visual reasoning, allowing us to determine whether highly specific editing instructions were accurately executed. For this evaluation process, we uniformly sample six corresponding frames from both the original source video and the final edited video. These paired frame sequences are simultaneously fed into GPT-5. This uniform sampling strategy provides the VLM with a comprehensive temporal view of the spatial transformations between the unedited and edited states, without exceeding the model's context window.

The GPT-5 evaluation is systematically conducted across four distinct dimensions to ensure a thorough assessment. Instruction Alignment (IA) measures the accuracy with which the model executed the target textual edit, such as a local object swap or attribute modification. Source Preservation (SP) evaluates the model's ability to retain the unedited background, original subjects, and overall context of the source video without introducing unintended alterations. Visual Quality (VQ) assesses the perceptual fidelity of the edited frames, looking specifically for spatial artifacts or structural distortions introduced during the generative process. Finally, the Overall (OVL) dimension provides a holistic score that measures the overall editing performance, acting as a proxy for comprehensive human judgment. Notably, the VLM only assesses the video in one aspect per each call, guaranteeing that different evaluation dimensions are independently measured. The wordings of the custom reasoning prompts utilized for this VLM-based assessment are provided in~\cref{prompt}

\section{Details on Human Evaluation}
\label{human}
\begin{figure}[!htbp]
    \centering
    \includegraphics[width=1\linewidth]{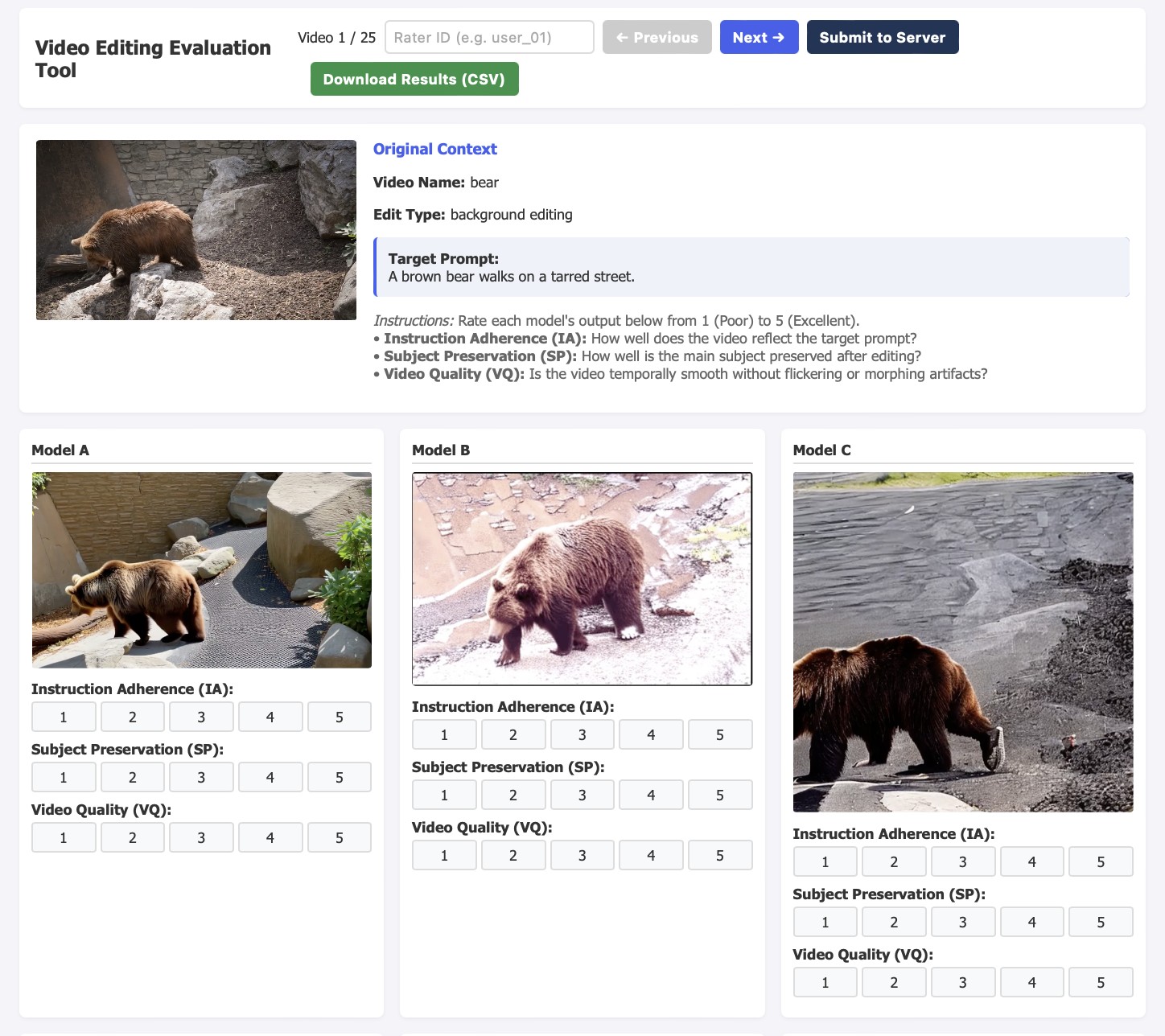}
    \caption{The user interface of our human evaluation website.}
    \label{fig:interface}
\end{figure}

\subsection{Evaluation Setup}
To conduct nuanced, human perceptual quality aligned evaluation, we conducted a comprehensive human evaluation study. To maintain a manageable cognitive load for our participants while ensuring diverse coverage, we randomly selected a testset of 25 unique video-prompt combinations.

We invite 10 independent human evaluators in this assessment. For each of the 25 evaluation instances, participants are presented with the source video, the driving text prompt, and the edited videos generated by the evaluated methods. To ensure a strictly blind evaluation and mitigate any potential cognitive or brand bias, all model identities are completely anonymized, and the generated outputs are randomly shuffled and assigned temporary identifiers ranging from Model A to Model G for each individual trial.

The evaluators are tasked with rating each generated video on a range from 1 to 5, across three critical perceptual dimensions. Video Quality (VQ) assessed the overall temporal smoothness, aesthetic appeal, and absence of generative artifacts. Instruction Adherence (IA) measures how accurately the edited video reflected the specific semantic changes requested in the text prompt. Finally, Source Preservation (SP) evaluates the model's ability to retain the unedited background, subject identity, and structural layout of the original source video. A custom user interface is designed specifically for this evaluation platform, which facilitated simultaneous side-by-side comparisons, is illustrated in~\cref{fig:interface}.

\begin{figure}[!htbp]
    \centering
    \includegraphics[width=1\linewidth]{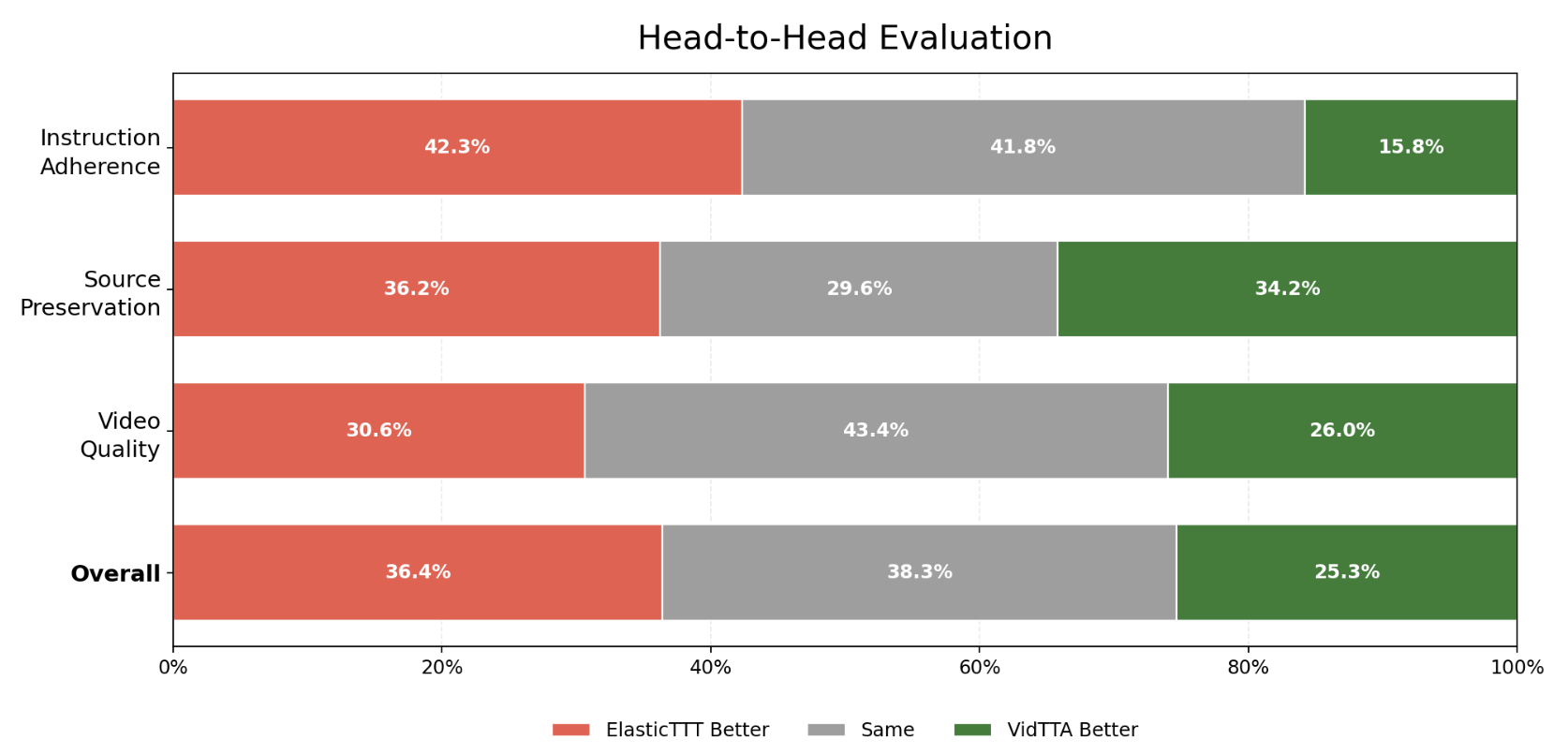}
    \caption{Head-to-head comparison results with the re-implemented VidTTA~\cite{wang2025low}.}
    \label{fig:head-to-head}
\end{figure}

\subsection{Head-to-head Human Evaluation with VidTTA}

As a highly competitive Test-Time Tuning (TTT) baseline that shares an identical foundational configuration with our proposed method, our re-implemented version of VidTTA~\cite{wang2025low} demonstrates remarkably strong performance on standard video editing tasks. Because the core optimization mechanisms of VidTTA closely parallel those of ElasticTTT, and the generation results demonstrate similar patterns, we deliberately exclude it from the broader multi-model human evaluation discussed previously, and explicitly conduct a \textit{head-to-head} human evaluation.

Evaluators are presented with a direct side-by-side comparison of videos generated by both ElasticTTT and VidTTA, conditioned on the same source video and driving text prompt. The positioning of the two models' outputs is entirely randomized for each trial, and their identities remained strictly anonymized. For every video pair, participants were asked to compare the generative outputs and cast a single vote indicating whether the first video or the second is better, or if both videos are of equal quality. These comparative judgments are made holistically, factoring in Video Quality, Instruction Adherence, and Source Preservation.

The quantitative results are illustrated in~\cref{fig:head-to-head}. The aggregated preference data clearly indicates that ElasticTTT consistently outperforms the VidTTA baseline across all evaluation metrics, securing a vastly higher win rate. Notably, ElasticTTT achieves a significant 42.3\% win rate over 15.8\% in Instruction Adherence, and 36.4\% over 25.3\% in Overall score. This significant margin confirms that ElasticTTT successfully overcomes the inherent limitations of test-time tuning, yielding visual outputs that are demonstrably preferred by human observers.

\begin{table}[!htbp]
  \caption{\textbf{Quantitative comparison with prior video editing approaches with 70 TTT steps} Best results are highlighted in \textcolor{red}{red}, and second-best in \textcolor{blue}{blue}. * refers that the method is re-implemented with our configurations.
  }
 
  \centering
  \label{tab:video_editing_comparison_70_0}
  \scriptsize
  \renewcommand{\arraystretch}{1.2}
  \setlength{\tabcolsep}{5pt}
  \begin{tabular}{lcccc}
    \toprule
    Methods &VQ$\uparrow$&IA$\uparrow$&SP$\uparrow$&OVL$\uparrow$ \\
    \midrule
    \multicolumn{5}{l}{\textit{Other Methods}} \\ 
     AnyV2V\cite{ku2024anyvv}  &3.51&5.89&2.30&4.31\\
     Ground-A-Video\cite{jeong2024groundavideo} & 3.57&4.49&2.42& 3.90\\
     Token-flow\cite{geyer2024tokenflow}  &4.47&5.57&4.30&4.83\\
    
     Ditto~\cite{bai2025ditto}&\textcolor{blue}{5.48}&5.95&3.32&\textcolor{blue}{5.62}\\
     Flow-align~\cite{kim2026flowalign}&5.28&5.44&\textcolor{red}{7.23}&5.44\\
     Uniedit~\cite{bai2025uniedit}  &3.87&\textcolor{blue}{6.51}&0.54&5.04\\
      \midrule
    \multicolumn{5}{l}{\textit{Test-Time-Tuning Methods}} \\ 
    Tune-A-Video*~\cite{wu2023tune}&5.75&6.67&3.62&5.80\\
    VidTTA*\cite{wang2025low}&5.16&6.06&3.86&5.27\\
    \rowcolor{yellow!30}ElasticTTT&\textcolor{red}{6.08}&\textcolor{red}{7.02}&\textcolor{blue}{4.68}&\textcolor{red}{6.68} \\

  \bottomrule
  \end{tabular}
    
\end{table}

\section{More Evaluation Results}
\label{more}
\subsection{Fewer Steps Evaluation}
Although we choose 100 TTT-steps as our final evaluation settings, we additionally evaluate ElasticTTT with 70 TTT-steps, comparing with the same group of baselines. Results are illustrated in~\cref{tab:video_editing_comparison_70_0}. Experiments have shown that 
our method can achieve state-of-the-art results but also able to maintain high level editing quality at smaller TTT steps.

\subsection{More TTT Steps with TDR}

Because the efficacy of Target Distribution Regularization (TDR) is deeply intertwined with the optimization dynamics of the Test-Time Tuning (TTT) process, we empirically evaluate the impact of TDR across varying optimization lengths. \cref{fig:linechart} illustrates the overall editing performance as a function of the number of TTT steps. As our theoretical analysis predicted, standard TTT exhibits a degradation in editing capability as the step count increases. In contrast, incorporating TDR consistently enhances overall editing performance across all evaluated TTT step counts, with especially significant improvement on large TTT steps, actively preventing the optimization process from collapsing.

\begin{figure}[!htbp]
    \centering
    \includegraphics[width=0.6\linewidth]{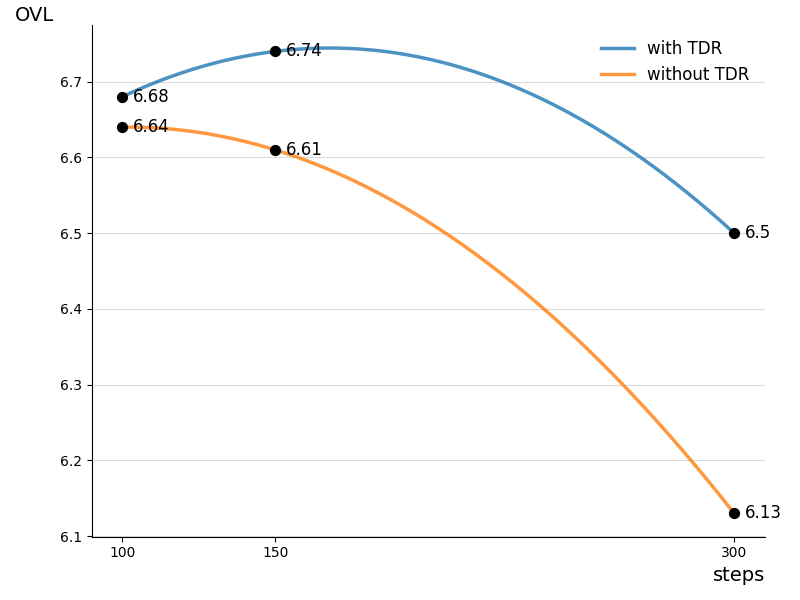}
   
    \caption{Comparison of Overall scores for different TTT steps, with and without TDR}
    \label{fig:linechart}
\end{figure}

To further quantify this phenomenon, we conduct an ablation study extending the optimization to an aggressive 300 TTT steps. As detailed in~\cref{tab:ablation-300}, removing TDR leads to a drastically compromised Video Quality (VQ), Instruction Adherence (IA) and Overall (OVL) score, demonstrating the severe conditioning collapse. The increase in Source Preservation illustrates the overfitting brought by TTT without the controlled stochasticity. The 300 steps experiment exhibits a largely more significant result than the ablation study in Sec. 4.3, as the model overfits step-by-step, exacerbating the prior collapse phenomenon gradually through training.
\begin{table}[!htbp]

  \caption{Quantitative results of TDR with 300 TTT steps. }
  
  \centering
  %   \scriptsize
  \renewcommand{\arraystretch}{1.2}
  \setlength{\tabcolsep}{5pt}
  %\vspace{-15pt}

  \label{tab:ablation-300}
  \begin{tabular}{lccccccc}
    \toprule
    Methods   & VQ$\uparrow$&IA$\uparrow$&SP$\uparrow$&OVL$\uparrow$\\
    \midrule
    w/o TDR  &6.09&6.28&\textcolor{red}{6.67}&6.13\\
  
    \rowcolor{yellow!30}Full (ElasticTTT) &\textcolor{red}{6.31}&\textcolor{red}{6.69}&6.38&\textcolor{red}{6.50} \\
    
  \bottomrule
  \end{tabular}
\end{table}

\subsection{Efficient Alignment-Editing Trade-off Via TDR}

In sensitive applications, particularly those involving human faces, source preservation is a paramount concern; even minor structural alterations can result in severely unnatural artifacts or loss of identity. While adjusting the TTT step count serves as an effective mechanism to control the degree of editing and source preservation, an inherent trade-off exists: enforcing stricter alignment with the reference video inevitably degrades semantic alignment and overall editing quality. Incorporating our TDR module significantly mitigates this issue, enabling robust reference alignment with minimal semantic sacrifice. Empirically, when increasing the TTT steps from 100 to 300, employing TDR yields a highly efficient trade-off: a 1-point drop in OVL translates to a substantial 5-point gain in SP. In contrast, without TDR, the identical semantic cost yields a mere 2.7-point improvement in SP.

\section{Correlation Analysis of VLM and Human Judgments} \label{sec:vlm and human correlation}
\label{corr}
%To prove the rationality of using the VLLM-based evaluation methd, we conduct two experiments to demonstrate the stability of the scores and the consistency of human judgment.
%Firstly, we prompts Chatgpt independently for five times to evaluate the stability of the score given by VLM. 

\begin{comment}
\begin{table}[!h]
  \caption{The stability test of VLM.
  }
  \label{tab:VLM stability}
  \centering
  \begin{tabular}{@{}llllll@{}}
    \toprule
    methods &Fisrt&Second& Third&fourth& fifth  \\

    \midrule
        Elastic TTT final score &6.68&6.61& 6.76&6.77& 6.57  \\

  \bottomrule
  \end{tabular}
\end{table}

\end{comment}

\begin{table}[!htbp]
  \caption{The $p$-values of two-sided Pearson correlation significance tests 
between human evaluation scores and VLM-given scores.  \textcolor{red}{Red} represent  \textbf{highly significant at the 0.01 level} and  \textcolor{blue}{blue} represent \textbf{statistically significant at the 0.05 level}.
  }
    \renewcommand{\arraystretch}{1.2}
  \setlength{\tabcolsep}{5pt}
  \label{tab:p value}
  \centering
  \begin{tabular}{@{}lllll@{}}
    \toprule
    methods &VQ&IA& SP&OVL  \\
    \midrule
      
     Flow-align~\cite{kim2025flowalign} & \textcolor{blue}{0.0172}&\textcolor{red}{0.0001}&\textcolor{red}{0.0005}&\textcolor{red}{0.0025}\\
     Ditto~\cite{bai2025ditto} &0.0747  &\textcolor{red}{0.0000} &0.0931&\textcolor{red}{0.0000}\\
      Uniedit~\cite{bai2025uniedit} &\textcolor{blue}{0.0116} &\textcolor{blue}{0.0126} &\textcolor{red}{0.0006}&\textcolor{blue}{0.0360}\\
     AnyV2V~\cite{ku2024anyvv} &0.0544  &\textcolor{red}{0.0005} &\textcolor{red}{0.0000}&\textcolor{blue}{0.0223}\\
     Ground-A-Video~\cite{jeong2024groundavideo} & 0.8595&\textcolor{red}{0.0000} &0.1202&0.3907\\
     Token-flow~\cite{geyer2024tokenflow} &\textcolor{red}{0.0070}&\textcolor{red}{0.0000} &\textcolor{blue}{0.0438} &\textcolor{red}{0.0000}\\

     ElasticTTT&\textcolor{red}{0.0099}&\textcolor{blue}{0.0220} &\textcolor{red}{0.0033}&\textcolor{red}{0.0038} \\
    
  \bottomrule
  \end{tabular}
\end{table}

\begin{table}[t]
\centering
\caption{Pearson correlation coefficient ($r$) between VLM-given
scores and human evaluation scores.
\colorbox{green!22}{\strut Strong} $r\geq0.5$;\;
\colorbox{orange!28}{\strut Moderate} $0.3\leq r<0.5$;\;
\colorbox{gray!18}{\strut Weak} $r<0.3$.}
  \renewcommand{\arraystretch}{1.2}
  \setlength{\tabcolsep}{5pt}
\label{tab:pearson_7models}
\begin{tabular}{lcccc}
\toprule
Methods & VQ & IA & SP & OVL \\
\midrule
  Flow-Align~\cite{kim2025flowalign}       & \cellcolor{orange!28}0.4041 & \cellcolor{green!22}0.7077 & \cellcolor{green!22}0.6310 & \cellcolor{green!22}0.5777 \\
  Ditto~\cite{bai2025ditto}            & \cellcolor{orange!28}0.3084 & \cellcolor{green!22}0.8686 & \cellcolor{orange!28}0.3146 & \cellcolor{green!22}0.7760 \\
  UniEdit~\cite{bai2025uniedit}          & \cellcolor{green!22}0.5265 & \cellcolor{orange!28}0.4945 & \cellcolor{green!22}0.6812 & \cellcolor{orange!28}0.4219 \\
  AnyV2V~\cite{ku2024anyvv}           & \cellcolor{orange!28}0.3999 & \cellcolor{green!22}0.6388 & \cellcolor{green!22}0.7474 & \cellcolor{orange!28}0.4595 \\
  Ground-A-Video~\cite{jeong2024groundavideo}   & \cellcolor{gray!18}$-$0.1498 & \cellcolor{green!22}0.6103 & \cellcolor{gray!18}0.2673 & \cellcolor{gray!18}0.1795 \\
  Token-Flow~\cite{geyer2024tokenflow}       & \cellcolor{green!22}0.5358& \cellcolor{green!22}0.7119& \cellcolor{orange!28}0.3899 & \cellcolor{green!22}0.7405 \\
  ElasticTTT       & \cellcolor{green!22}0.5922 & \cellcolor{orange!28}0.4245 & \cellcolor{green!22}0.5907 & \cellcolor{green!22}0.5915 \\
\bottomrule
\end{tabular}
\end{table}

\begin{figure}[!htbp]
    \centering
    \includegraphics[width=1\linewidth]{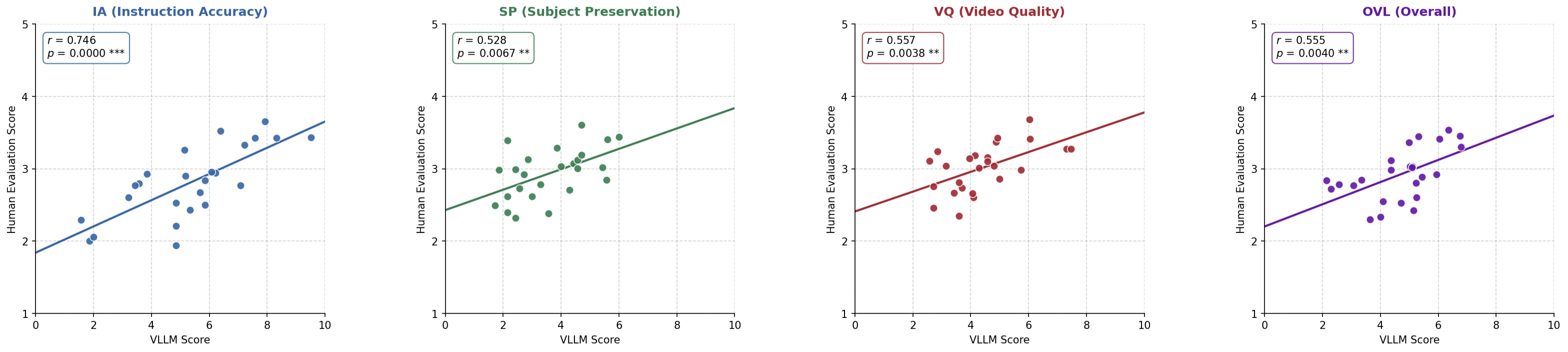}
   
    \caption{Correlation visualization between VLM score vs. Human Evaluation}
    \label{fig:human-vllm-pearon-comparison}
\end{figure}

To rigorously assess the agreement between our automated Vision-Language Model (VLM) evaluation and human perceptual judgments, we compute the Pearson correlation coefficient ($r$) across four core dimensions: Instruction Adherence (IA), Source Preservation (SP), Video Quality (VQ), and the Overall score (OVL). For a sample size of $n = 25$ video clips, the correlation coefficient between the VLM-assigned scores ($x_i$) and the human evaluation scores ($y_i$) is calculated as:
\begin{align}
    r = \frac{\sum_{i=1}^{n}(x_i - \bar{x})(y_i - \bar{y})}
             {\sqrt{\sum_{i=1}^{n}(x_i - \bar{x})^2} \cdot 
              \sqrt{\sum_{i=1}^{n}(y_i - \bar{y})^2}}
\end{align}
where $\bar{x}$ and $\bar{y}$ denote the respective sample means. This coefficient, bounded within $[-1, 1]$, quantifies the strength of the linear relationship between the two scoring paradigms.

To evaluate the statistical significance of these observed correlations, we formulate a two-sided hypothesis test. Let $\rho$ denote the true population correlation coefficient. We define our null hypothesis against the alternative hypothesis as follows:
\begin{align}
    H_0: \rho = 0 \quad \text{vs.} \quad H_1: \rho \neq 0
\end{align}
To assess the evidence against $H_0$, which posits that any observed correlation is strictly due to random sampling variance, we compute the $t$-statistic:
\begin{align}
    t = r \sqrt{\frac{n - 2}{1 - r^2}}
\end{align}
Under the null hypothesis, this test statistic follows a Student's $t$-distribution with degrees of freedom $df = n - 2$. For our evaluation sample, $df = 23$. Based on this $t$-statistic, we calculate the corresponding two-sided $p$-value, representing the probability of observing a correlation as extreme as the computed $r$ entirely by chance (assuming $H_0$ holds). A sufficiently small $p$-value provides strong evidence to reject $H_0$ in favor of $H_1$, indicating a statistically significant linear relationship. The specific $p-values$ and significant levels are demonstrated in~\cref{tab:p value}. The specific Pearson correlation coefficients for each of the seven evaluated models are reported in~\cref{tab:pearson_7models}, with categorization thresholds adopted from Cohen's established guidelines~\cite{1988Statistical}.

Furthermore, we visualize this alignment in~\cref{fig:human-vllm-pearon-comparison}. Each data point in the scatter plot represents an individual editing task, with its spatial coordinates corresponding to the mean VLM and human evaluation scores averaged across all seven models.
In the figure caption, $r$ denotes the correlation coefficient and $p$ denotes the corresponding $p$-value for statistical significance testing. An asterisk (*) indicates that the correlation is statistically significant, typically with $p < 0.05$. More asterisks (*) represent a higher significance.

Ultimately, these quantitative analyses reveal a strong, positive, and statistically significant linear correlation between VLM scoring and human evaluation in the vast majority of cases. This robust alignment demonstrates that our VLM-based metrics serve as a highly reliable and scalable proxy for human visual judgment.

\section{Prompt Design for VLM Judgment}
To ensure the reproducibility of the results, we list the prompt for GPT-5 evaluation below:
\label{prompt}

\begin{lstlisting}[language=json, caption={VLM prompt for Overall Score}]
"You will be given two frame sequences:
- Sequence A: original video frames
- Sequence B: edited video frames
You are scoring a video editing result from 0 to 10. 
Please assign scores more dispersedly based on the actual situation
in the video.

Scoring rubric:
1) Edit fulfillment (highest weight):
does edited video follow edited prompt vs original prompt?
2) Content preservation: 
keep unrelated content/identity/background when they should remain.
3) Temporal and visual consistency: 
stable motion, no obvious artifacts/flicker.

Return strict JSON only with keys:
{"score_0_10": number, "reason": string, "confidence": number}
Where score_0_10 in [0,10], confidence in [0,1], reason <= 80 words."
\end{lstlisting}
\begin{lstlisting}[caption={VLM prompt for Video Quality}]
"You will be given two frame sequences:
- Sequence A: original video frames
- Sequence B: edited video frames
You are scoring a video editing result from 0 to 10.
Please assign scores more dispersedly based
on the actual situation in the video.

"You are scoring a edited video quality from 0 to 10. 
Please assign scores more dispersedly based 
on the actual situation in the video."
        "Scoring rubric:"
        "1) image quality: 
        the image from the edited video is of high quality, 
        should be sharp and clear."
        "2) Aesthetic score: 
        the edited video is aesthetically pleasing, 
        the motion is smooth and natural."
        "3) Background consistency:
        the background is consistent with the original video,
        no obvious artifacts/flicker."
        "4) subject consistency: 
        the object is consistent with the original video, 
        no obvious artifacts/flicker."
        "5) motion smoothness: 
        the motion is smooth and natural,
        no obvious artifacts/flicker."
        "Return strict JSON only with keys:"
        '{"score_0_10": number, "reason": string, "confidence": number}'
        "Where score_0_10 in [0,10], confidence in [0,1], 
        reason <= 80 words."

Return strict JSON only with keys:
{"score_0_10": number, "reason": string, "confidence": number}
Where score_0_10 in [0,10], confidence in [0,1], 
reason <= 80 words."
\end{lstlisting}

\begin{lstlisting}[caption={VLM prompt for Video}]

"You will be given two frame sequences:
- Sequence A: original video frames
- Sequence B: edited video frames
You are scoring a video editing result from 0 to 10. 
Please assign scores more dispersedly based 
on the actual situation in the video.

"You are scoring a video editing result from 0 to 10. 
Please assign scores more dispersedly based
on the actual situation in the video."
        "Scoring rubric:"
        "1) Semantic alignment: 
        does edited video follow edited prompt vs original prompt?"

        "Return strict JSON only with keys:"
        '{"score_0_10": number, "reason": string, "confidence": number}'
        "Where score_0_10 in [0,10], confidence in [0,1],
        reason <= 80 words."
Return strict JSON only with keys:
{"score_0_10": number, "reason": string, "confidence": number}
Where score_0_10 in [0,10], confidence in [0,1], reason <= 80 words."
\end{lstlisting}

\begin{lstlisting}[caption={VLM prompt for Source Preservation}]
"You will be given two frame sequences:
- Sequence A: original video frames
- Sequence B: edited video frames
You are scoring a video editing result from 0 to 10. 
Please assign scores more dispersedly based
on the actual situation in the video.

"You are scoring a video editing result from 0 to 10. 
Please assign scores more dispersedly based 
on the actual situation in the video."
        "Scoring rubric:"
        "1) How much the details in the background
        (the place that should not be edited) are consistent 
        with the original video?"
        
        '{"score_0_10": number, "reason": string, "confidence": number}'
        "Where score_0_10 in [0,10], confidence in [0,1], 
        reason <= 80 words."
\end{lstlisting}
\section{Demonstration of Asynchronous Noise Scheduling Masks}
\begin{figure}[!htbp]
    \centering
\vspace{-1pt}    \includegraphics[width=1.02\linewidth]{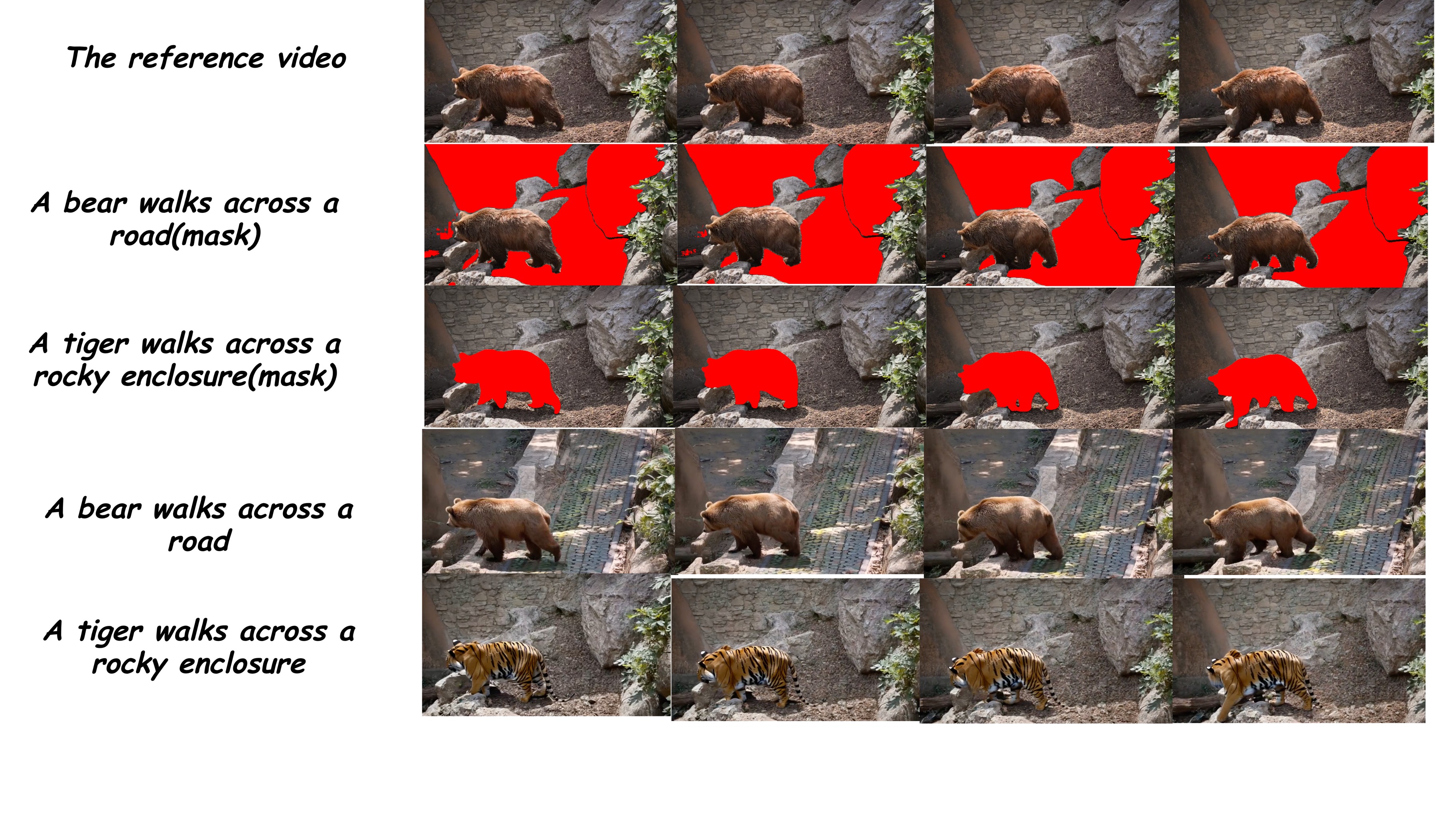}
    \caption{Visual illustration of our mask area and the final edited video.}
    \label{fig:mask}
 
\end{figure}
\label{async}
In~\cref{fig:mask}, we illustrate the mask used for Asynchronous Noise Scheduling (Async-NS) generated by Grounded-SAM2~\cite{ravi2024sam2}. 

We showcase two different editing scenarios of the same reference video to further demonstrate the relationship between editing prompt and mask area. The second and third row show the masked areas and the fourth and fifth row are our final editing results. 

\section{Limitations and Future Work}
\label{limit}
\subsection{Limitation}
Although our model can achieve state-of-the-art performance in all editing tasks, it still suffers from a few limitations. 
Firstly, the TTT nature of our model leads to more computation cost compared to those training-free methods.
Secondly, in some situation, we found that over-specification of a certain area in the video can also limit the editing capabilities of other areas. We infer that this phenomenon may be caused by the over-optimization of the self-attention layers during the Test-Time-Tuning process.

\subsection{Future Work}
Currently, our Asynchronous Noise Scheduling relies on explicit masks generated by external models like Grounded-SAM2. Future work could eliminate this dependency by dynamically extracting spatial masks directly from the diffusion model's internal cross-attention maps during the initial TTT steps, creating a fully end-to-end framework. 

Also, we aim to extend ElasticTTT to handle long-context or multi-shot videos. By introducing temporal chunking or sliding-window TTT, we can ensure temporal consistency across hundreds of frames without exceeding standard VRAM constraints.

\section{More Demonstrations}
\label{demo}
In~\cref{fig:placeholder34}, we further show two challenging cases with increased motion and complex postures.
In~\cref{fig:additional_comparisons}, we demonstrate two more cases comparing to other video editing methods.

\begin{figure}[!htbp]
    \centering
    \includegraphics[width=1.0\linewidth]{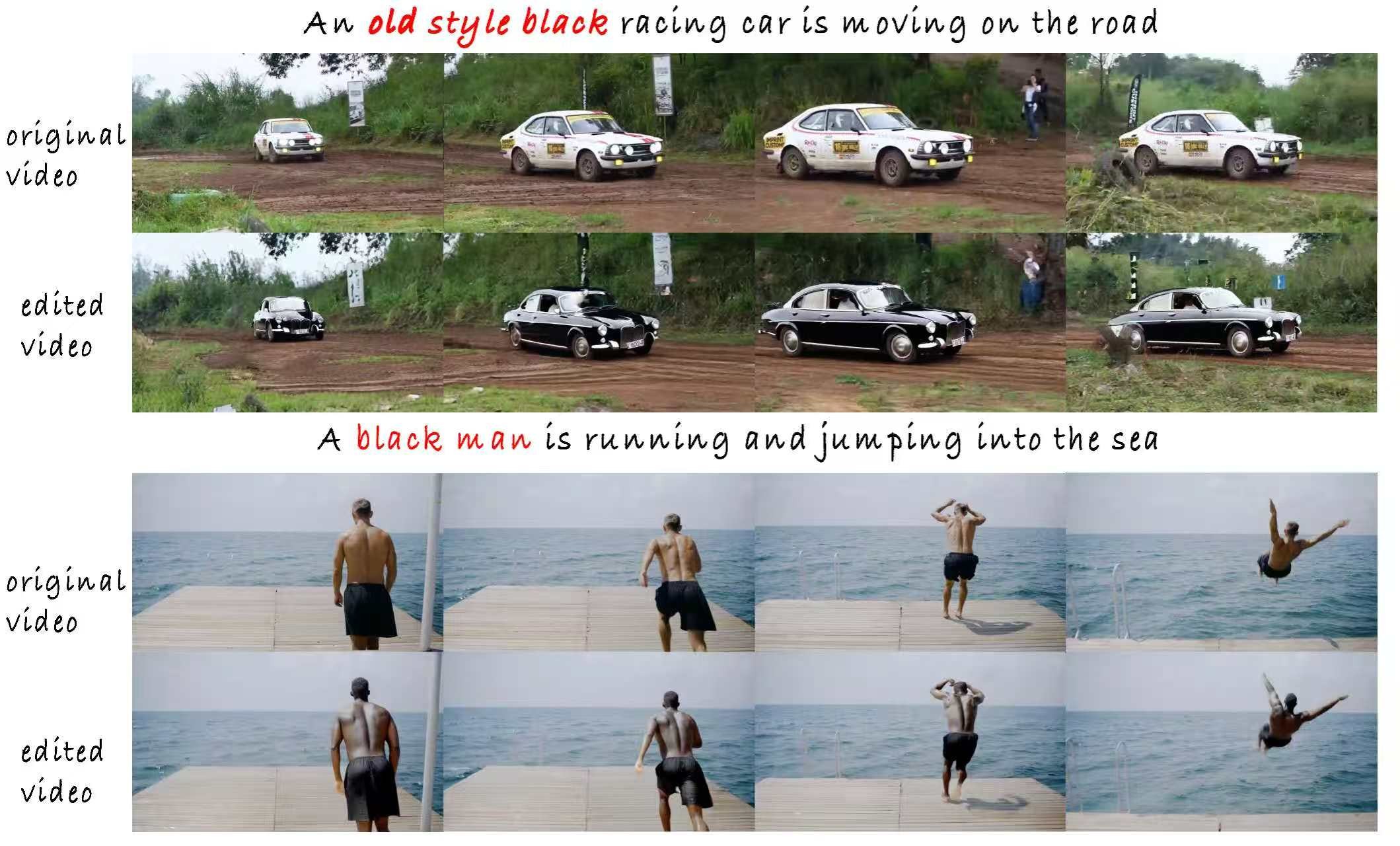}
    \caption{Two more challenging cases.}
    \label{fig:placeholder34}
\end{figure}

\newpage

\begin{figure}[!htbp]
    \centering
    \includegraphics[width=1\linewidth]{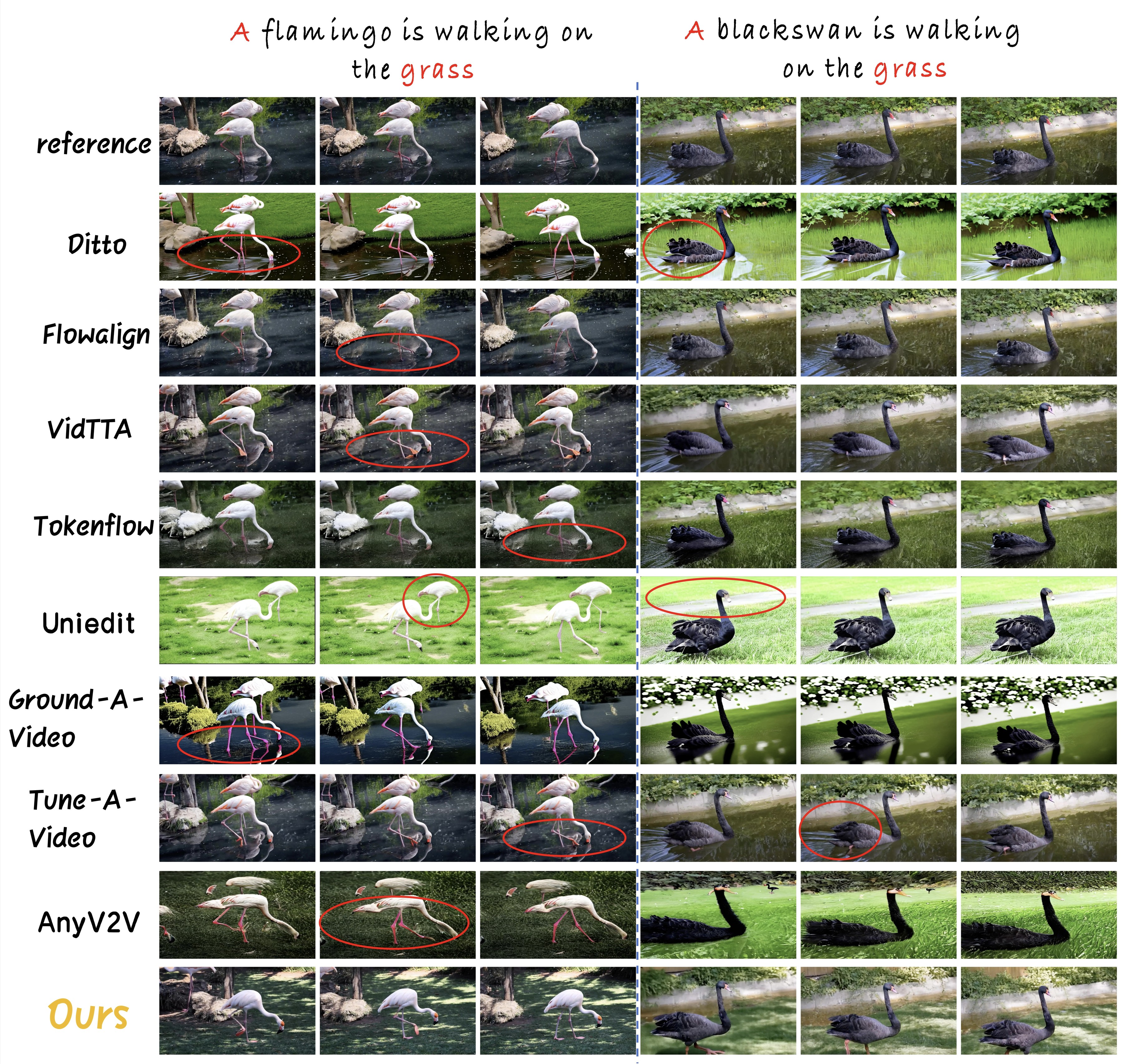}
    \caption{Visual comparisons with other methods. Our method strictly follows the editing instructions while successfully preserving the details in unedited regions.}
    \label{fig:additional_comparisons}
 
\end{figure}

\section{Social Impact and Ethical Concerns}
\label{social}
\subsection{Social Impact}
\subsubsection{Positive Social Impact}
ElasticTTT serves as an accessible and highly effective framework that democratizes video editing for a broad user base. By enabling intuitive modifications solely through natural language instructions, it eliminates the steep learning curve traditionally associated with professional-grade software such as DaVinci Resolve. Furthermore, our model offers substantial value to professional creators by automating routine editing tasks, thereby significantly accelerating workflows and reducing both temporal and financial production costs.
\subsubsection{Potential Risk}
Although our proposed ElasticTTT democratizes high-quality video editing, it inherently carries risks related to malicious video synthesis. The ability to precisely edit specific regions of a video could be misused to create convincing counterfeit content, such as placing individuals in fabricated scenarios without their consent. This contributes to the broader societal challenge of deepfakes and misinformation propagation. We strictly condemn the use of our method for generating deceptive content. To safeguard against such misuse, we suggest that downstream applications of our framework incorporate invisible watermarking and pair the technology with advanced synthetic media detection algorithms.

\subsection{Ethical Concerns for Human Evaluation}
Our human evaluation framework is designed in strict compliance with ethical research standards. First, all evaluators engaged in the assessment strictly on a voluntary basis, retaining the unconditional right to opt out at any stage without facing any negative repercussions. Second, we enforced a rigorous data privacy policy; all collected feedback was entirely anonymized, ensuring that no personally identifiable details were ever recorded. Furthermore, the study is purely observational. The core task merely requires human subjects to review standard synthesized videos, which completely shields them from physical hazards, psychological stress, or any offensive visual stimuli. Consequently, our methodology avoids any invasive interventions that might compromise participant well-being.

\subsection{Safeguard}
It is quite essential to implement proper safeguards when developing user-friendly video editing software or platform. To ensure a thoroughly vetted and secure assessment environment for our human evaluation, we strictly sampled videos from our curated dataset, which is constructed based on the renowned DAVIS benchmark. This deliberate selection process guarantees that all visual content is benign, thereby actively preventing the inadvertent inclusion of inappropriate, sensitive, or unvetted media commonly found on the internet.

\newpage
% \bibliographystyle already set in main file, do not override
% \bibliographystyle{splncs04}
% \bibliography main already called in main file

\end{document}